%% file: main.tex
\documentclass{article} 
\usepackage{paperstyle,times}
\usepackage{tikz}

\usepackage{amsmath}
\usepackage{amssymb}
\usepackage{url}
\usepackage{enumitem}
\usepackage{wrapfig}
\usepackage{graphicx}
\usepackage{tcolorbox}
\usepackage{subfig}
\usepackage{mathtools}
\usepackage{amsthm}
\usepackage{algorithm}
\usepackage{algpseudocode}
\usepackage{booktabs}
\usepackage{hyperref}
\usepackage{multirow}
\usepackage{xspace}

\input{math_commands.tex}

\input{macros.tex}

\title{\model: Spectral conjugate for Spectral Graph Convolutional Networks}

\author{%
  Chun Hei Michael Chan \hspace{4.8em} Flavia Petruso \hspace{4.8em} Dimitri Van De Ville \\[0.6em]
  Neuro-X Institute and Department of Electrical Engineering \\
  École polytechnique fédérale de Lausanne (EPFL) \\
  Geneva, 1202, Switzerland \\[0.4em]
}

\begin{document}


\maketitle

\begin{abstract}

Graph convolutional networks propagate information by repeated local aggregation through a graph shift operator; i.e., a $K$-layer network reaches $K$ hops neighborhood. On the one hand, such spreading can lead to oversmoothing. On the other hand, long-range dependencies demand the depth. Transporting information on long distances and without attenuation requires the shift to distinguish a direction of flow, which a symmetric operator cannot perform but a directed one can fulfill. A natural way to extract pure-directionality is to take the skew-symmetric part of the shift operator through the Cartesian split, which, however, generally does not commute with the shift itself, meaning that the filters built on it are not shift-invariant. We instead use the spectral conjugate; i.e., the image of the operator under $\tau:z\mapsto \bar{z}$, which commutes with the shift and splits it into a dissipative and a non-dissipative part. Two filter families follow: a sum filter, whose non-dissipative component transports signal without energy loss, and a \filtername filter, \filtername in the pair of components rather than polynomial in the shift. Both arise from non-holomorphic kernels, placing them outside the holomorphic class underlying classical spectral convolution. On the directed cycle, the \filtername filter becomes an IIR filter with global impulse response, for which we prove a long-range reach gap against every degree-$K$ polynomial filter. Chebyshev reparameterization gives stable vertex-domain layers with real coefficients, yielding \model, which solves graph transfer tasks at reduced depth and is competitive with state-of-the-art graph convolutional networks on node classification benchmarks.

\end{abstract}

\section{Introduction}
\label{sec:introduction}

Graph convolutional networks are among the most widely used architectures for learning on relational data, owing to their efficiency, being light weight and to their increased interpretability given the spectral theory that underpins them~\citep{bruna2014spectral, defferrard2016convolutional, kipf2017semisupervised}. However, two limitations are well documented: node representations tending to converge toward a constant as depth increases, known as oversmoothing~\citep{li2018deeper, oono2020graph, cai2020note, rusch2023survey} and a $K$-layer network aggregating over at most $K$ hops neighborhood, which requires sufficient depth to capture long-range dependencies~\citep{topping2022understanding, di2023over}. Residual connections, normalization, and rewiring~\citep{li2019deepgcns, alon2021on} serve as effective solutions for both obstacles, yet, because our goal is to address their common cause, we do not compare against this line of work.

The reasons for these limitations are largely related to the nature of the operator used in the convolutional networks. For a symmetric graph operator, high frequencies tend to be attenuated due to its real spectrum, leading to progressively less distinctive features since spectral convolution strengthens this attenuation when the number of convolution layers (depth) increases~\citep{oono2020graph, cai2020note}. Escaping this requires eigenvalues on the imaginary axis, an established criterion for stable non-dissipative dynamics \citep{haber2018stable, chang2019antisymmetricrnn}, which A-DGN and its successors obtain by imposing skew-symmetry on the learnable weights \citep{gravina2023adgn, heilig2025porthamiltonian,hariri2026return}. For directed graphs, the graph shift already possesses non-zero imaginary part as it is non-symmetric~\citep{singh2016graph, marques2020signal}.

A natural way to extract skew-symmetry and thus imaginary spectrum is the Cartesian split, which separates the shift operator into its symmetric and skew-symmetric parts. We show, however, that neither part commutes with the shift unless the graph is normal---which directed graphs generally are not~\citep{trefethen2005spectra, asllani2018structure}---so the induced filters are not shift-invariant. To overcome this limitation while still obtaining imaginary spectrum, we instead use the spectral conjugate~\citep{nevanlinna2018non}, which commutes with the shift on any graph, coincides with the transpose when the shift is normal, and remains defined for non-diagonalizable shifts. On this decomposition, we build learnable graph convolutional layers from filters in~\citep{chan2026graph}. Composing the resulting layers gives the \model framework. Importantly, since the underlying scalar functions (kernels) are not complex-differentiable, the construction falls outside the holomorphic functional calculus classical spectral filters rely on \citep{higham2008functions}, introducing a new class of spectral graph convolutions.


\paragraph{Contributions.}
(1) We propose the spectral conjugate, rather than the transpose, to split the directed Laplacian~\citet{singh2016graph} into dissipative and non-dissipative components while preserving shift invariance. (2) We construct the sum and \filtername filters from the resulting components, prove they are both non-holomorphic, and derive their vertex-domain forms via the Cauchy-Pompeiu formula. (3) We prove that on the directed cycle the \filtername filter is an IIR filter with global impulse response, and establish a reach gap against every degree-$K$ polynomial filter in the shift. (4) We reparameterize both filters in Chebyshev bases, yielding stable vertex-domain layers with real coefficients. (5) \model solves graph transfer tasks at a fraction of the depth conventional polynomial filters require, preserves feature distinctiveness with depth, and is competitive with state-of-the-art GCNs on general benchmarks for node classification.

\section{Dissipative, Asymmetric, and Spectral Conjugate Filters}
\label{sec:motivation}

\subsection{Why polynomial spectral filters attenuate}

Let $G=(\mathcal{V},\mathcal{E},\ma W)$\footnote{Notations are provided in Appendix~\ref{app:notations}} be a weighted graph with $N=|\mathcal{V}|$ nodes and weight matrix $\ma W\in\mathbb{R}^{N\times N}_{\geq0}$, and let $\ma L\in\mathbb{R}^{N\times N}$ be the corresponding graph shift operator~\citep{ortega_introduction_2022} with spectrum $\sigma(\ma L)=\{\lambda_n\}\subset\mathbb{C}$. Denote its Jordan decomposition as $\ma L=\ma P(\ma\Lambda+\ma N)\ma P^{-1}$, where $\operatorname{diag}(\ma\Lambda)=\sigma(\ma L)$ and $\ma N$ is the nilpotent part. The standard spectral filter is the matrix $h(\ma L)$ induced by a scalar kernel $h$ that is analytic on a neighborhood $\mathcal{D}\subset\mathbb{C}$ of $\sigma(\ma L)$, on the Jordan block $\ma J_n$ of size $m$ attached to $\lambda_n$,
\begin{equation}
    h(\ma J_n)=\sum_{k=0}^{m-1}\frac{h^{(k)}(\lambda_n)}{k!}\,\ma N_n^{k},
    \label{eq:jordan-filter}
\end{equation}
placing $h(\lambda_n)$ on the diagonal and its complex derivatives $h^{(k)}(\lambda_n)$ on the successive superdiagonals~\citep{sandryhaila_discrete_2014}. Filtering a signal reads as $\vc y=h(\ma L)\vc x$, and $l$ layers apply $h(\ma L)^{l}$. Since $\ma N_n$ is nilpotent, the sum terminates, so the gain of each spectral component is $|h(\lambda_n)|$ up to a factor polynomial in $m$. The standard GCN layer can be read as the degree one polynomial $h(\lambda)=1-t\lambda$ \citep{chamberlain2021grand}, the first-order truncation of the heat kernel $h(\lambda)=e^{-t\lambda}$ whose gain is upper bounded by $e^{-t\Re(\lambda_n)}$
(Appendix~\ref{app:heat-kernel-non-diag}). Components with $\Re(\lambda_n)>0$ thus decay geometrically with depth and only those with $\Re(\lambda_n)=0$ survive: as $l\to\infty$ the representation collapses onto the $\Re(\lambda)=0$ invariant subspace, which is the spectral form of oversmoothing. The following observation shows that under a symmetric shift this subspace is uninformative.

\paragraph{Symmetric shifts and non-dissipative propagation.} A real symmetric $\ma S$ has $\sigma(\ma S)\subset\mathbb{R}$, so $\Re(\lambda_n)=\lambda_n$ for every $n$; and for a Laplacian on a connected graph the eigenvalue $0$ is simple with eigenvector $\mathbf{1}$ \citep{fiedler1973algebraic, chung1997spectral}. Under a symmetric shift, the constant signal is therefore the only mode with $\Re(\lambda)=0$, and the only one preserved under arbitrary depth. Real symmetric operators therefore do not allow for non-dissipative propagation. For directed graphs, the operator is asymmetric, enables purely imaginary eigenvalues, and consequently allows non-dissipative propagation.


\subsection{What about the Cartesian split}
Specify $\mathcal{G}$ to be a directed graph, and consider the directed Laplacian $\ma L=\ma D-\ma W$ with real weight matrix $\ma W$ and $\ma D=\operatorname{diag}(\ma W\mathbf{1})$ the in-degree matrix~\citep{singh2016graph}. $\ma L$ is non-symmetric, and its spectrum is complex with $\Im(\lambda)\neq0$ in general. To isolate the asymmetric structure, a natural approach is the Cartesian split $\ma L=\tfrac{1}{2}(\ma L+\ma L^T)+\tfrac{1}{2}(\ma L-\ma L^T)$, whose second term is skew-symmetric and thus has purely imaginary spectrum. Besides pure imaginary spectrum, it is crucial for both terms to be linear shift-invariant (LSI) --- to commute with the shift operator~\citep{sandryhaila_discrete_2014} --- to enable the analysis of graph signals in the same frequency domain, and thereby the interpretation of the features. Additionally  filters based on either components are LSI w.r.t.\ to $\ma L$ only when they the components themselves are.

\begin{prop}[Cartesian split breaks shift invariance]
\label{prop:cartesian-commute}
The Cartesian components $\tfrac{1}{2}(\ma L\pm\ma L^{T})$ commute with $\ma L$ if and only if $\ma L$ is normal.
\end{prop}
Even when diagonalizable, directed Laplacians are usually non-normal~\citep{trefethen2005spectra, asllani2018structure}. Hence the Cartesian split yields components that are not LSI. Thereby, a different split satisfying commutation with $\ma L$ with summands respectively carrying spectrum on real $\sigma_R(\ma L)=\{\Re(\lambda_n)\}$ and imaginary axis $\sigma_{I}(\ma L)=\{j\Im(\lambda_n)\}$, is desirable. This leads us to naturally consider the spectral conjugate~\citep{nevanlinna2018non}. 

\subsection{Non-holomorphic functional and the spectral conjugate}
The spectral conjugate of $\ma L$ is the matrix induced by the non-analytic kernel $\tau:z\mapsto \bar{z}$~\citep{nevanlinna2018non}, non-analytic since its Wirtinger derivative $\bar{\partial}\tau:= \partial \tau/\partial \bar{z}=1$ (Appendix~\ref{app:wirtinger-holomorphicity-poly-ratio}). Yet in~\eqref{eq:jordan-filter}, $h$ is required to be analytic, equivalently holomorphic. Therefore to evaluate $\tau$, we leverage non-holomorphic functional calculus. First recall that holomorphic functional calculus~\citep{auscher1997holomorphic} presents the evaluation of an analytic function $h$ at $\lambda \in \Omega$ via the Cauchy integral
\begin{equation}
    h(\lambda)
    = \frac{1}{2\pi i}\oint_{\Gamma} \frac{h(z)}{z - \lambda}\,dz,
    \label{eq:contour-holomorphic}
\end{equation}
where $\Gamma$ is any contour enclosing $\lambda$ in its interior. Lifting this to the matrix setting results in the standard graph filters (Appendix~\ref{app:holomorphic-to-graphfilters}) and classical spectral graph filters relying on canonical, Chebyshev, or Cayley polynomials falls within this framework (Appendix~\ref{app:wirtinger-holomorphicity-poly-ratio}). Evaluating $\tau$ instead follows from the Cauchy-Pompeiu formula~\citep{bell2015cauchy}, which extends the Cauchy integral to smooth but non-holomorphic functions:
\begin{equation}
    h(\lambda)
    = \frac{1}{2\pi i}\oint_{\Gamma} \frac{h(z)}{z-\lambda}\,dz + \frac{1}{2\pi i}\iint_{\Omega} \frac{\bar\partial h(z)}{z-\lambda}\,dz\wedge d\bar z,
    \label{eq:contour-non-holomorphic}
\end{equation}
where $\Omega$ is the region enclosed by $\Gamma$. Evaluating $\tau$ at $\ma L$ through~\eqref{eq:contour-non-holomorphic}, as in~\citet{nevanlinna2018non}, yields the subsequent definition of spectral conjugate. 

\begin{definition}\cite[Def.~2.5 -- Spectral conjugate]{nevanlinna2018non}
\label{def:spectral-conjugate}
The spectral conjugate of $\ma L$ under the non-holomorphic function $\tau: z\mapsto \bar{z}$ is given by
\begin{equation}
    \ma L^{\sharp}:= \tau(\ma L) = \ma P\tau(\ma J)\ma P^{-1} = \ma P\overline{\ma \Lambda} \ma P^{-1},
    \label{eq:spectral-conjugate}
\end{equation}
where $\tau(\ma J)$ is composed of the Jordan blocks $\tau(\ma J_n) = \overline{\lambda}_n\ma I$.
\end{definition}




\begin{theorem}[Conjugate properties]
\label{thm:conjugate-properties} 
$\ma L^{\sharp}$ commutes with $\ma L$, is real, and coincides with $\ma L^{T}$ if and only if $\ma L$ is normal.
\end{theorem}

\begin{corollary}[Shift-invariant spectral split]
\label{cor:split}
The decomposition
\begin{equation}
    \ma L=\underbrace{\tfrac{1}{2}\bigl(\ma L+\ma L^{\sharp}\bigr)}_{\ma L^{\circ}}
    +\underbrace{\tfrac{1}{2}\bigl(\ma L-\ma L^{\sharp}\bigr)}_{\ma L^{\uparrow}}
    \label{eq:spectral-split}
\end{equation}
is exact, both components commute with $\ma L$ and with each other, are real, and $\sigma(\ma L^{\circ})=\sigma_R(\ma L)$, $\sigma(\ma L^{\uparrow})=\sigma_I(\ma L)$. For normal $\ma L$, it reduces to the Cartesian split. Thereby we term $\ma L^{\circ}$ and $\ma L^{\uparrow}$  dissipative and asymmetric operators, respectively.
\end{corollary}


In practice, $\ma L^{\sharp}$ is computed with the Schur--Parlett algorithm~\citep{parlett1976recurrence,davies2003schur}, which evaluates a matrix function on the Schur form instead of the (numerically unstable) Jordan form. Its cost is dominated by the Schur decomposition which scales as $O(N^3)$.

\subsection{Spectral filters: the sum and \filtername kernels}
\label{sec:non-holomorphicity}
Recalling~\eqref{eq:spectral-split}, one can design kernels that operate separately on $\sigma_R(\ma L)$ and $\sigma_I(\ma L)$. We consider the two kernels of~\citep{chan2026graph}, 
\paragraph{Sum kernel.}
Separate polynomial expansions in the real and imaginary parts,
\begin{equation}
  h_{\sumf}(z)=\sum_{k=0}^{K_1}\vc a[k]\,\left(\frac{z+\bar{z}}{2}\right)^k \;+\;\sum_{k=0}^{K_2}\vc b[k]\,\left(\frac{z-\bar{z}}{2}\right)^k, \qquad \vc a[k],\vc b[k]\in\mathbb{C}.
  \label{eq:kernel-sum}
\end{equation}
\paragraph{Ratio kernel.}
Consider the polynomial in the ratio $(z-\bar{z})/(z+\bar{z})$,
\begin{equation}
  h_{\ratf}(z)= \vc c[0] + 
  \begin{cases}
    \displaystyle\sum_{k=1}^{K}\vc c[k]
      \left(\dfrac{z-\bar{z}}{z+\bar{z}}\right)^{\!k}, & \Re(z)\neq 0,\\[10pt]
    0, & \Re(z)=0,
  \end{cases}
  \qquad \vc c[k]\in\mathbb{C},
  \label{eq:kernel-rational}
\end{equation}
The case $\Re(z)=0$ arises only at the direct current (DC) eigenvalue $\lambda=0$. Setting $h_{\ratf}(0)=\vc c[0]$ removes the division-by-zero instability, and allows filtering of the DC component. 
We next show that the sum and \filtername~(\eqref{eq:kernel-sum}, \ref{eq:kernel-rational}) kernels span a new function class, being non-holomorphic. 

\begin{theorem}[Vertex-domain sum and \filtername filters]
\label{thm:non-holomorphicity}
The kernels $h_{\sumf}$ and $h_{\ratf}$ are non-holomorphic. Their corresponding graph filters admit the vertex-domain representations
\begin{align} \label{eq:graph-sum-filter}
    h_{\sumf}(\ma L)
    = \sum_{k=0}^{K_1}\vc a[k]\,\bigl(\ma L^{\circ}\bigr)^k + \sum_{k=0}^{K_2}\vc b[k]\,\bigl(\ma L^{\uparrow}\bigr)^k, \quad h_{\ratf}(\ma L) = \sum_{k=0}^{K}\vc c[k]\,\bigl(\ma L^{\uparrowcirc}\bigr)^k,
\end{align}
where $\ma L^{\uparrowcirc} = \ma L^{\uparrow}(\ma L^{\circ})^{D}$ and $(\ma L^{\circ})^{D}$ the Drazin inverse (Appendix~\ref{def:drazin}) of $\ma L^{\circ}$.
\end{theorem}


Theorem~\ref{thm:non-holomorphicity} has two consequences. First, it grounds the vertex-domain forms~\eqref{eq:graph-sum-filter} for any $\ma L$, diagonalizable or not, whereas~\citet{chan2026graph} define them only in the diagonalizable case. Second, no polynomial or convergent power series in $z$ represents either kernel. On a fixed graph, $h(\ma L)$ still equals a polynomial in $\ma L$ interpolating $h$ on $\sigma(\ma L)$, but one that is graph-specific, of degree up to $N$, and ill-conditioned (Appendix~\ref{app:finite-polynomial-interpolation}). The new class is thus distinguished at the kernel level, and at the matrix level by degree and conditioning.

\section{Spectral graph convolution networks}
\subsection{Graph filter design}
Canonical polynomial implementations of spectral filters are ill-conditioned in their coefficients, and Chebyshev bases are the typical choice to tackle this limitation~\citep{boyd2001chebyshev}; however, they apply to real spectra. We show that~\eqref{eq:graph-sum-filter} admits a Chebyshev reparameterization nonetheless. Throughout, $T_k$ denotes the degree $k$ Chebyshev polynomial, and is recursively defined as
\begin{equation}
   T_k(x)=2x\,T_{k-1}(x)-T_{k-2}(x),\quad k\ge 2,
   \qquad T_0(x)=1,\ T_1(x)=x.
  \label{eq:cheb_rec}
\end{equation}

\paragraph{Sum filter.}
The operator $\Lc$ has real eigenvalues $\{\lR[n]\}$, which we rescale to $[-1,1]$:
\begin{equation}
  \wLc=\frac{\Lc-c_R\ma I}{R_R},\qquad
  c_R=\tfrac{1}{2}\bigl(\max_n\lR[n]+\min_n\lR[n]\bigr),\qquad
  R_R=\max_n|\lR[n]-c_R|+\varepsilon .
  \label{eq:scale-diff}
\end{equation}
The operator $\Lu$ has purely imaginary eigenvalues $\{j\lI[n]\}$, so $-j\Lu$ has real spectrum $\{\lI[n]\}$. Because $\Lu$ is a real-valued matrix, its eigenvalues occur in conjugate pairs, giving $c_I=0$; no recentering is needed and
\begin{equation}
  \wLu=-j\,\Lu/R_I,\qquad R_I=\max_n|\lI[n]|+\varepsilon .
  \label{eq:scale-adv}
\end{equation}
Combining degree-$K_1$ and degree-$K_2$ expansions in the two rescaled operators yields
\begin{equation}
  h_{\sumf}(\ma L)=\sum_{k=0}^{K_1}\vct\alpha[k]\,T_k(\wLc) + \sum_{k=0}^{K_2}\vct\beta[k]\,T_k(\wLu),
  \label{eq:cheb-sum}
\end{equation}
with learnable $\vct\alpha\in\mathbb{C}^{K_1+1}$, $\vct\beta\in\mathbb{C}^{K_2+1}$.

\paragraph{Ratio filter.}
The quotient operator $\ma L^{\uparrowcirc}=\Lu(\Lc)^{D}$ likewise has purely imaginary eigenvalues $j\lI[n]/\lR[n]$ for $\lR[n]\neq 0$, so the same substitution applies. Setting $\ma Z=-j\ma L^{\uparrowcirc}$ gives the real spectrum $\vct z[n]=\lI[n]/\lR[n]$, occurring in conjugate pairs, hence $c_Q=\tfrac{1}{2}\bigl(\max_n \vct z[n]+\min_n \vct z[n]\bigr)=0$ and
\begin{equation}
  \widetilde{\ma Z}=\ma Z/R_Q,\qquad R_Q=\max_n|\vct z[n]|+\varepsilon ,
  \label{eq:scale-rat}
\end{equation}
so that $h_{\ratf}(\ma L)=\sum_{k=0}^{K}\vct\eta[k]\,T_k(\widetilde{\ma Z})$ with learnable $\vct\eta\in\mathbb{C}^{K+1}$. To observe the effect of reparameterization, we provide the frequency response for each filter $h_{\sumf}(\ma L), h_{\ratf}(\ma L)$ in Appendix~\ref{app:freq-response-cheb}.
\begin{theorem}[Real-valued twins filters]
\label{thm:real_matrix} 
Real-valued filters expressible by $h_{\sumf}(\ma L)$ are exactly those obtained with $\vct\alpha[k]\in\R$, $\vct\beta[2k]\in\R$ and $\vct\beta[2k+1]\in j\R$ for all $k$. Likewise, real-valued filters expressible by $h_{\ratf}(\ma L)$ are exactly those obtained with  $\vct\eta[2k]\in\R$ and $\vct\eta[2k+1]\in j\R$ for all $k$.
\end{theorem}

Despite the odd-order coefficients being purely imaginary, factoring $j$ out for these terms brings back the problem to a real-valued optimization. This effectively handles concerns of memory and compute overhead when optimizing over complex-valued parameter space~\citep{kramer2024tutorial}.


\subsection{The \model architecture}
\label{sec:architecture}
In what follows, we describe the layer design for both the sum and \filtername filters.
\paragraph{Sum filter based layer design.}
We propose a simple layer design that leverages the sum graph filter proposed in~\eqref{eq:graph-sum-filter}. Denote $\ma X^{(l)}$ the node features of dimension $N \times C$ (nodes times channels) at the $l$-th layer. The $(l+1)$ layer update reads as:
\begin{equation}
    \label{eq:TwinSpecGCN-sum}
    \ma X^{(l+1)} = \psi\left(\nu \sum_{k=0}^{K_1} T_k(\wLc)\ma X^{(l)}\ma \Theta_k^{(l, \circ)} + (1-\nu)\sum_{k=0}^{K_2}j^k\, T_k(\wLu)\ma X^{(l)}\ma \Theta_k^{(l, \uparrow)}\, + \ma B^{(l)}\right),
\end{equation}
where $\psi$ is a pointwise nonlinear function, $\ma \Theta_{k}^{(l,\circ)}$ and $\ma \Theta_{k}^{(l,\uparrow)}$ are learnable real weight matrices, $\ma B^{(l)}$ is a real learnable bias term and $\nu\in[0,1]$ is real learnable parameter as introduced in \citet{rossi2024edge}, controlling importance of dissipative and non-dissipative parts. Composing sum filter based layers, we build the \modelS model.
\paragraph{Ratio filter based layer design.}
Similarly, the update rule for the $(l+1)$-th layer, based on the \filtername graph filter proposed in~\eqref{eq:graph-sum-filter} is given by:
\begin{equation}
    \label{eq:TwinSpecGCN-ratio}
    \ma X^{(l+1)} = \psi\left(\sum_{k=0}^{K} j^k\, T_k(\widetilde{\ma Z})\ma X^{(l)}\ma \Theta_k^{(l)}\, + \ma B^{(l)}\right),
\end{equation}
where $\ma \Theta_{k}^{(l)}$ is a learnable real weight matrix. Consecutive \filtername filters constitute the \modelR model. We also propose to combine both models leading to \modelC (Appendix~\ref{app:TwinSpecGCN-combined}).

\subsection{Theoretical analysis}
\label{subsec:theoretical-analysis}
\subsubsection{Mitigating oversmoothing}
\label{subsubsec:oversmoothing}
Oversmoothing arises when repeated graph convolution drives node features toward a constant vector, and is diagnosed by the decay across layers of the Dirichlet energy~\citep{rusch2023survey}
\begin{equation}
    \mathcal{S}(\ma X^{(l)})
    = \tfrac{1}{2}\!\!\sum_{(n,m)\in\bar{\mathcal{E}}}\!\!
      \bigl\|\ma X^{(l)}_n - \ma X^{(l)}_m\bigr\|^2,  \qquad \bar{\mathcal{E}} = \mathcal{E}\cup\{(m,n):(n,m)\in\mathcal{E}\}.
    \label{eq:dirichlet-energy}
\end{equation}
As shown by the sum in~\eqref{eq:dirichlet-energy}, the Dirichlet energy sees only the symmetrized edge set $\bar{\mathcal{E}}$, or in another words the symmetric part of the shift, making it a measure of feature distinctiveness rather than of directed smoothness~\citep{marques_signal_2020}. We retain it in that sense, following \citet{choi2023gread}, and verify the mitigation empirically in Section~\ref{sec:receptive-oversmoothing} by tracking $\mathcal{S}(\ma X^{(l)})$ across depth. 

We analyze this in particular on the directed cycle $\mathcal{C}_N$, where $\ma L$ is normal and the graph Fourier basis is unitary, so for a scalar signal $\vc x^{(l)}$ Parseval identity gives
\begin{equation}
    \mathcal{S}(\vc x^{(l)})
    = \bigl(\vc x^{(l)}\bigr)^{\!T}\!\ma L^{\circ}\,\vc x^{(l)}
    = \sum_{n=0}^{N-1}\lR[n]\,\bigl|\hat{\vc x}^{(l)}[n]\bigr|^2 .
    \label{eq:dirichlet-cycle}
\end{equation}

Since $\Lu$ has purely imaginary spectrum $\sigma_I(\ma L)$, any spectral function of unit modulus applied to it leaves every $|\hat{\vc x}[n]|$ unchanged and hence preserves~\eqref{eq:dirichlet-cycle} exactly. The filter $e^{\beta\Lu}$ relies on such function, which it acts as $\hat{\vc x}[n]\mapsto e^{j\beta\lI[n]}\hat{\vc x}[n]$, a pure phase rotation with $\mathcal{S}(e^{\beta\ma L^{\uparrow}}\vc x)=\mathcal{S}(\vc x)$. A degree one sum filter with $\vc a=(0,0)$ and $\vc b=(1,\beta)$ realizes its first-order truncation,
\begin{equation}
    \vc x^{(l+1)}_{\sumf} = \bigl(\ma I+\beta\ma L^{\uparrow}\bigr)\vc x^{(l)}
    \;\Longrightarrow\;
    \mathcal{S}\bigl(\vc x^{(l+1)}_{\sumf}\bigr) = \sum_{n}\lR[n]\bigl(1+\beta^{2}\lI[n]^{2}\bigr) \bigl|\hat{\vc x}^{(l)}[n]\bigr|^{2},
    \label{eq:energy-first-order}
\end{equation}
so the truncation never contracts energy.

\begin{wrapfigure}{r}{0.48\linewidth}
  \vspace{-14pt}
  \centering
  \includegraphics[width=\linewidth]{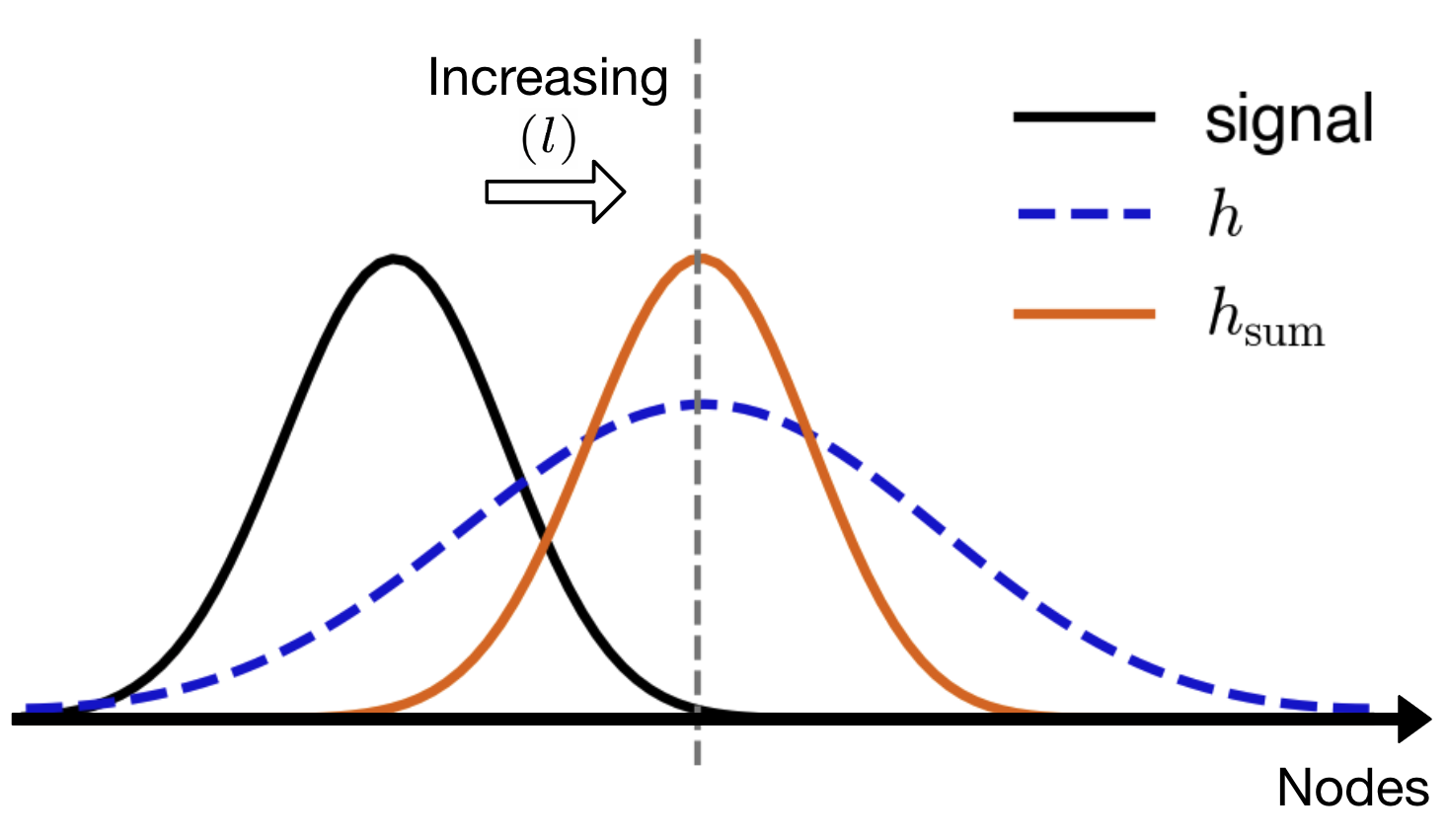}
  \caption{Isolating $\Lu$ allows $h_\sumf$ to translate the signal while the polynomial kernel $h$ inseparably moves and flattens it.}
  \label{fig:oversmoothing}
  \vspace{-12pt}
\end{wrapfigure}

The mechanism is the complementary spectral action of the two components: $\ma L^{\circ}$ has real non-negative spectrum and attenuates high-frequency modes, whereas $\ma L^{\uparrow}$ transports features by phase alone and does not attenuate. Therefore increasing the depth does not necessarily imply smoothing. The argument extends beyond $\mathcal{C}_N$. On any directed graph, the eigenvalues of $e^{\beta\Lu}$ are $e^{j\beta\lI[n]}$, all of unit modulus, so no mode is attenuated exponentially: for diagonalizable $\ma L$, the filter is a pure phase rotation of each Fourier coefficient (no attenuation), and for defective $\ma L$, the nilpotent part adds at most a factor polynomial in $\beta$~(Appendix~\ref{app:phase-kernel-non-diag}). 

\subsubsection{Expanding receptive field reach}

Message-passing networks built on a local shift operator share a structural ceiling: a $K$-layer network moves information at most $K$ hops (Theorem~\ref{thm:reach-gap} (i)), so long-range aggregation often demands large depth. We show that the quotient operator $\ma L^{\uparrowcirc}$ alleviates this ceiling, taking $\mathcal{C}_N$ with a pulse-translation task as a setting that isolates receptive field reach.

\paragraph{Graph \filtername and IIR filters.}
The operator $\ma L^{\uparrowcirc}$ shares the eigenbasis of the shift but has eigenvalues $\vct\mu[n] = j\lI[n]/\lR[n]$. On $\mathcal{C}_N$ these are a Möbius transform of the shift eigenvalues $\vct z[n] = e^{j\vct \theta[n]}$, which turns the graph \filtername filter into a \filtername function of the shift.

\begin{theorem}[IIR correspondence]
\label{thm:fir-iir}
On $\mathcal{C}_N$ with shift eigenvalues $\vct z[n]=e^{j\vct \theta[n]}$, the eigenvalues $\vct \mu[n]$ satisfy the relation $\vct \mu[n] =(\vct z[n]+1)/(\vct z[n]-1)$. Consequently, a degree-$K$ polynomial filter in $\ma L^{\uparrowcirc}$ has spectral response
\begin{equation}
\sum_{k=0}^{K}\vc c[k]\,\vct \mu[n]^k
\;=\;
\frac{\sum_{k=0}^{K}\vc c[k]\,(\vct z[n]+1)^k(\vct z[n]-1)^{K-k}}{(\vct z[n]-1)^{K}}
\;=\;\frac{Q_K(\vct z[n])}{(\vct z[n]-1)^{K}},
\label{eq:rational}
\end{equation}
i.e. a \filtername or infinite impulse response (IIR) filter of type $[K/K]$ in the shift, with all $K$ poles located at the DC mode $z=1$. Its vertex-domain impulse response therefore has global support.
\end{theorem}

Although IIR for $\mathcal{C}_N$, on general graphs the \filtername filter is no longer a \filtername function of $\ma L$. Different from graph IIR filters~\citep{isufi2016autoregressive, liu2018filter}, it needs no iterative solver to perform filtering and inserts directly into any convolutional layer, retaining finite impulse response (FIR) efficiency while achieving IIR reach.

\paragraph{The reach of the \filtername filter.} We measure how far a filter propagates information. For $\vc y = h(\ma L)\vc x$ with $\vc x$ an impulse at node $0$, the influence of node $0$ on node $d$ (at distance $d$\,) is $h[d,0]:=h(\ma L)[d,0] = \partial\vc y[d]/\partial\vc x[0]$ (the sensitivity of~\citet{di2023over}), and we define the reach as the relative sensitivity
\begin{equation}
  S_h(d) = \frac{|h[d,0]|}{\max_{v\in\mathcal{V}}|h[v,0]|} \in [0,1].
  \label{eq:reach}
\end{equation}

\begin{theorem}[Reach of shift, heat, and \filtername filters]
\label{thm:reach-gap}
On $\mathcal{C}_N$ with nodes $\mathcal{V}=\{0,\dots,N-1\}$, let $d\in \mathcal{V} \setminus \{0\}$, and assume $N\geq 3K+2$. Then:
\begin{enumerate}
\item[(i)] \emph{(Shift filter)} If $h=\sum_{m=0}^{K}\gamma_m\ma{W}^m$, then $h[d,0]=\gamma_d$ for $d\leq K$ and 
\begin{equation}
  S_h(d)=0, \quad \text{for every } d>K.
\end{equation}
\item[(ii)] \emph{(Heat filter)} If $h=e^{-t\Lc}$ with $t>0$, then 
\begin{equation}
\ma S_h(d) \le\ \frac{(t/2)^d}{d!}.
\label{eq:heatbound}
\end{equation}
\item[(iii)] \emph{(Ratio filter)} If $h=h_\ratf:=\sum_{k=0}^{K}\vc c[k] (\ma L^{\uparrowcirc})^k$ then for $K<d<N-K$
\begin{equation}
  S_h(d) \approx \frac{|p(d/N)|}{\max_{v\in\mathcal{V}}|p(v/N)|}\, , \quad p(s)=\sum_{k=1}^{K}\vc c[k]\,\frac{(-1)^{k-1}2^{k}}{k!}\,N^{k-1}B_k(s),
\label{eq:ratsens}
\end{equation}
where $B_k(s)$ is the Bernoulli polynomial of degree $k$.
\end{enumerate}
\end{theorem}

\begin{wrapfigure}{r}{0.5\linewidth}
  \vspace{-8pt}
  \centering
  \includegraphics[width=1\linewidth]{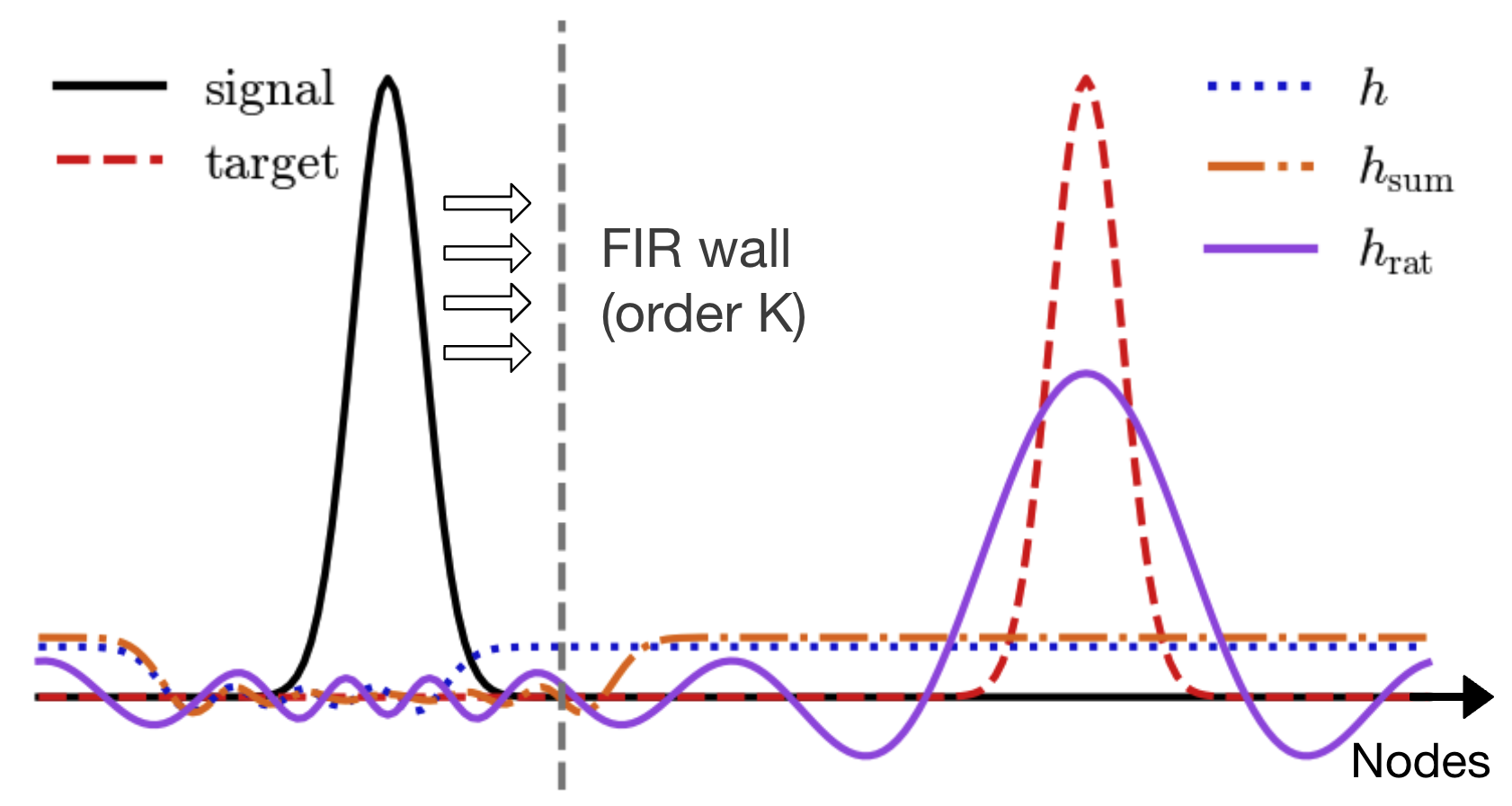}
  \caption{FIR models are obstructed by FIR wall whereas graph \filtername filter overcomes it on $\mathcal{C}_N$.}
  \label{fig:reach-schematic}
  \vspace{-8pt}
\end{wrapfigure}

The three regimes are qualitatively distinct. Reach of the shift polynomial (FIR) is identically zero beyond $K$ hops (what we refer to as the FIR wall shown in Figure~\ref{fig:reach-schematic}), the heat kernel (IIR) decays super-exponentially, and the \filtername filter only polynomially. At $K=1$, \eqref{eq:ratsens} gives $S_h(d)=|1-2d/N|$, nonzero at every distance except for $d=N/2$. Considering $\mathcal{G}(K)=\max_{h_{\ratf}}\min_{d\in\mathcal{V}}S_h(d)$, i.e., the worst-case reach at order $K$, we obtain $\mathcal{G}(K)>0.95$ for $K\geq4$ (Figure~\ref{app:reach-numerical}). 
While the $[K/K]$ characterization and bound on sensitivity are specific to $\mathcal{C}_N$, the escape from the $K$-hop class remains valid for general directed graph as the filters are not FIR. We empirically show this on general graphs through the graph transfer task~\citep{bodnar2021weisfeiler, di2023over} in Section~\ref{sec:receptive-oversmoothing}, and further experiments reported in Appendix~\ref{app:reach-generalization}. 

\section{Related Works}
\label{sec:related-works}
\paragraph{Spectral graph convolutional networks.}
FaberNet (HoloNet) \citep{koke2024holonets} performs spectral filtering on directed graphs by approximating holomorphic functions through Faber polynomials \citep{coleman1987faber, ellacott1983computation}, outperforming MagNet \citep{zhang2021magnet}, a directed GCN built on the Magnetic Laplacian \citep{furutani2020graph}. Both constructs filter directly from a graph operator without separating the mechanisms governing propagation, and in particular without the dissipative and asymmetric split of \eqref{eq:spectral-split}. This is reflected in the function class they span: being holomorphic, neither can represent the kernels of Theorem~\ref{thm:non-holomorphicity}. Dir-GNN \citep{rossi2024edge} aggregates separately over in- and out-neighborhoods with distinct weight matrices. Since $(\ma L,\ma L^T)\mapsto\bigl(\tfrac{1}{2}(\ma L+\ma L^T),\tfrac{1}{2}(\ma L-\ma L^T)\bigr)$ is bijective, this realises a Cartesian split, and Proposition~\ref{prop:cartesian-commute} applies: the components do not commute with the shift, so the filters are not LSI.

\paragraph{Advection-diffusion graph neural networks.}
For diagonalizable shifts, \eqref{eq:spectral-split} coincides with the graph diffusion and advection operators of~\citet{chan2026graph} who study them as signal-processing tools, without learning or propagation analysis. This places our construction alongside architectures of similar concept \citep{eliasof2024feature, wu2023supercharging}, which we include as baselines. Further conceptual comparisons with these networks are provided in Appendix~\ref{app:advdiff}.

\section{Experiments}
\label{sec:experiments}
Baselines include methods discussed in Section~\ref{sec:related-works}: MagNet, FaberNet, Dir-GNN, ADR-GNN and AdvDIFFormer. Methods that instead constrain the weight parametrization \citep{gravina2023adgn, heilig2025porthamiltonian} (discussed in Section~\ref{sec:introduction}) act on a complementary axis and can be combined with ours (Appendix~\ref{app:anti-symmetric weights}), so we treat them as orthogonal rather than competing.

\subsection{Receptive field reach and Oversmoothing}
\label{sec:receptive-oversmoothing}

\begin{figure}[H]
    \centering
    \captionsetup[subfigure]{justification=centering}
  \subfloat[Accuracy as a function of the number of layers \label{fig:long-range-mse-depth}]{%
       \includegraphics[width=0.98\linewidth]{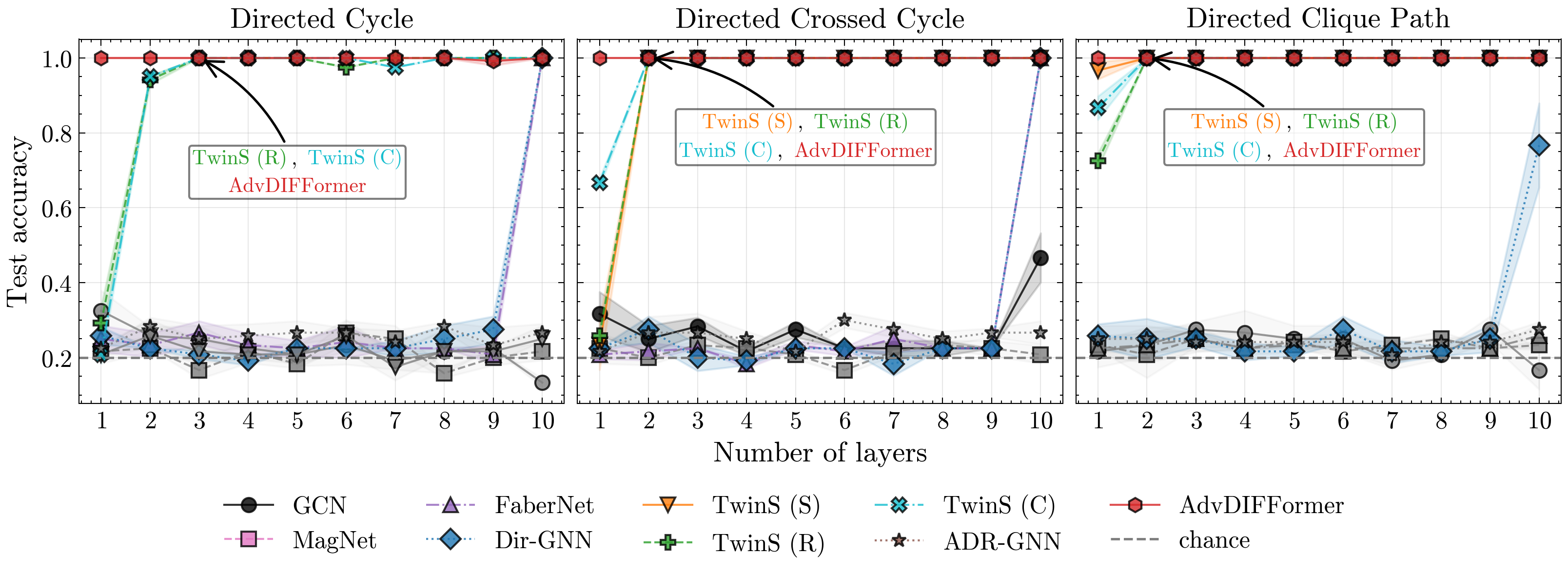}}
       \hfill \\
    \subfloat[Feature energy at each layer \label{fig:long-range-mse-depth-oversmoothing}]{%
       \includegraphics[width=0.97\linewidth]{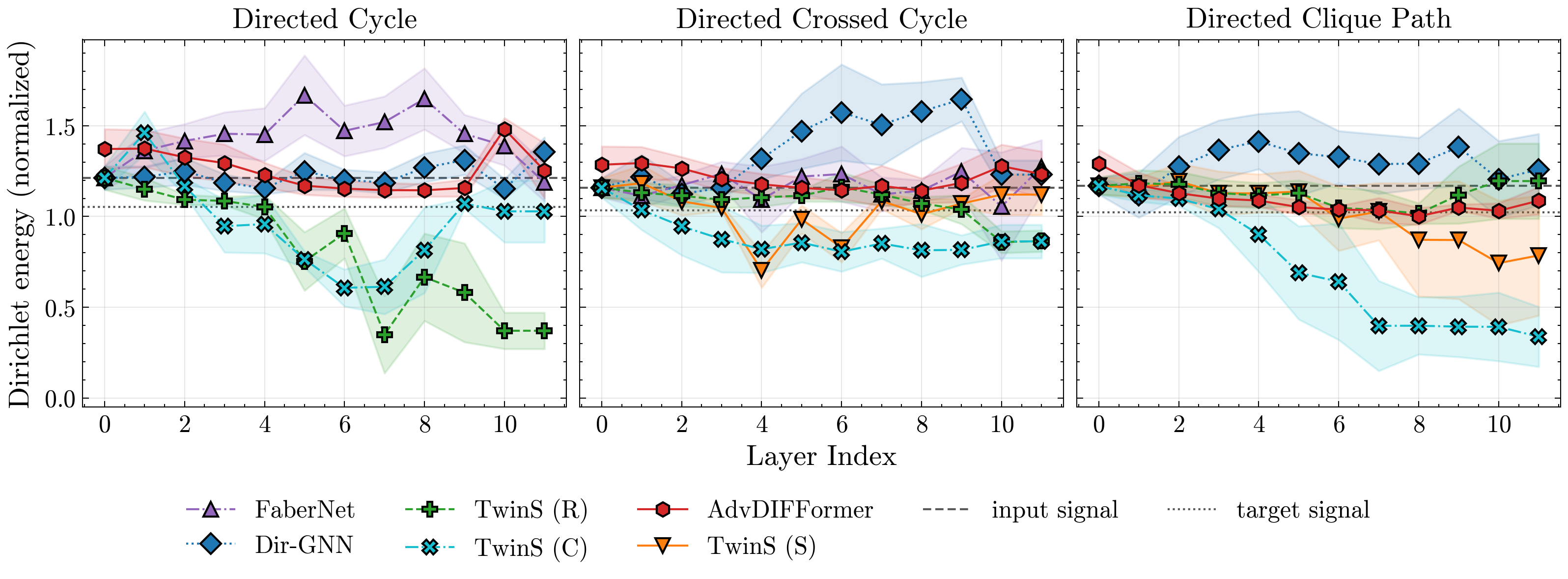}}
  \caption{a) Accuracy of different models on the graph transfer task. In the box are models with overlapping accuracy of $1$. Models with performance lower than chance level ($20\%$) at all depths are displayed in gray. b) Dirichlet energy computed after each convolution layer before activation, averaged within batch. Only models ($10$ layers) with accuracy beyond chance level are shown. Input and target signal's energy are also shown for each graph. In both case, error bars are over trials.}
\end{figure}

\paragraph{Settings.} We evaluate reach and oversmoothing jointly on the graph transfer task \citep{bodnar2021weisfeiler, di2023over}, adapted to directed graphs (Appendix~\ref{app:exp_details}). The task involves the propagation of a feature vector from a source node to a target node, which are separated by a certain number of hops in the graph. Following~\citet{choi2023gread}, we track the Dirichlet energy (\eqref{eq:dirichlet-energy}) of the node features across layers as a measure of feature distinctiveness. 

\paragraph{Results.} 
From Figure~\ref{fig:long-range-mse-depth}, \modelR and \modelC transfer the feature accurately with only a few layers on all three graphs, and \modelS does so on the crossed cycle and clique path, demonstrating \model's ability to capture long-range dependencies at minimal depth. AdvDIFFormer also solves the task at low depth, but at a higher cost in parameters and runtime (Table~\ref{tab:runtime-graph-transfer-best}). The remaining polynomial methods (MagNet, FaberNet, Dir-GNN, GCN) and ADR-GNN instead require a depth comparable to the distance to target, or does not solve the task. Figure~\ref{fig:long-range-mse-depth-oversmoothing} shows that on each graph at least one \model variant maintains its Dirichlet energy across layers, preserving feature distinctness while solving the task. These results are consistent with reach and non-attenuation properties of Section~\ref{subsec:theoretical-analysis} beyond directed cycle.

\subsection{Benchmark datasets}
\paragraph{Settings.} We evaluate performances on  directed node-classification benchmarks: Chameleon, Squirrel (filtered version as in~\citet{platonov2023critical}) representing links between Wikipedia pages related to chameleons and squirrels \citep{rozemberczki2021multi}, Texas, Wisconsin, and Cornell which are datasets showing links between websites of different universities \citep{pei2020geomgcn}. We also compare our model against MLP~\citep{rumelhart1986learning} and undirected graphs baselines --- GCN~\citep{kipf2017semisupervised} and GAT~\citep{brody2021attentive}.

\paragraph{Results.} The results of the benchmark experiments are summarized in Table~\ref{tab:benchmark-results}. Our proposed \model model matches the baseline methods across the datasets with best overall rank despite lower accuracy on Squirrel-F, demonstrating its effectiveness in capturing the underlying directed structures and improving node classification accuracy. Reasons for the increased performance could be attributed to the properties of long-range reach and mitigation of oversmoothing, as discussed in Section~\ref{sec:receptive-oversmoothing}, along with the class of function our model leverages.

\begin{table}[H]
\centering
\caption{Node classification accuracy (in percentage) across benchmark datasets. Best result per dataset in \textbf{bold}, second best \underline{underlined}. Additional information in Appendix~\ref{app:datasets}.}
\label{tab:benchmark-results}
\resizebox{0.92\columnwidth}{!}{%
\begin{tabular}{llcccccc}
\toprule
& \textbf{Model}
  & \textbf{Chameleon-F}   & \textbf{Squirrel-F}
  & \textbf{Texas}    & \textbf{Wisconsin} & \textbf{Cornell} & \textbf{Mean Rank}\\
\midrule
\multirow{3}{*}{\rotatebox[origin=c]{90}{\small }}
& MLP                &  $36.73\pm4.71$ & $36.55\pm1.82$ & $82.16\pm4.34$ & $85.16\pm4.95$ & $\underline{75.86}\pm\underline{3.94}$ & $6.70$ \\
\midrule
\multirow{3}[-4]{*}{\rotatebox[origin=c]{90}{\small Undir.}}
& GCN                &  $40.89\pm4.12$ & $39.47\pm1.47$ & $52.16\pm5.16$ & $48.92\pm3.06$ & $52.70\pm5.30$ & $10.20$\\
& GAT                & $ 39.21\pm3.08$ & $35.62\pm2.06$ & $58.38\pm6.63$ & $49.41\pm4.09$ & $54.32\pm5.05$ & $10.20$\\
\midrule
\multirow{3}[6]{*}{\rotatebox[origin=c]{90}{\small Dir.}}
& MagNet    & $42.36\pm3.53$ & $44.19\pm1.39$ & $\underline{85.31}\pm\underline{4.56}$ & $83.40\pm3.98$ & $73.33\pm4.17$ & $\underline{4.20}$\\
& Dir-GNN    & $44.13\pm2.92$ & $\underline{44.28}\pm\underline{2.17}$ & $ 70.54\pm6.52$ & $67.39\pm6.00$ & $59.82\pm6.90$ & $6.40$\\
& FaberNet    & $42.99\pm 3.33$ & $\textbf{45.13}\pm\textbf{2.06}$ & $83.15\pm5.30$ & $79.54\pm 5.18$ & $72.61\pm4.72$ & $4.60$\\
& AdvDIFFormer & $41.51\pm3.63$ & $41.41\pm1.59$ & $81.80\pm6.36$ & $76.54\pm4.84$ & $72.16\pm5.50$ & $7.20$ \\
& ADR-GNN            & $\textbf{46.07}\pm\textbf{4.07}$ & $44.09\pm1.97$ & $82.16\pm5.16$ & $84.51\pm4.54$ & $71.35\pm4.39$ & $4.90$\\
\midrule
\multirow{3}{*}{\rotatebox[origin=c]{90}{\small Ours}}
& \modelS       & $42.33\pm2.80$ & $ 39.61\pm1.55$& $84.59\pm3.43$ & $\underline{87.12}\pm\underline{4.13}$ &  $ 75.14\pm3.15$ & $4.60$\\
& \modelR       & $\underline{44.24}\pm\underline{3.23}$ & $40.15\pm1.83$ & $84.05\pm3.91$ & $\textbf{87.39}\pm\textbf{3.98}$ & $\textbf{76.76}\pm\textbf{3.01}$ & $\textbf{2.80}$  \\
& \modelC       & $41.64\pm3.59$ & $40.10\pm1.58$ & $\textbf{85.41}\pm\textbf{4.71}$ & $86.80\pm3.79$ & $75.68\pm4.01$ & $\underline{4.20}$ \\
\bottomrule
\end{tabular}%
}
\end{table}

\section{Conclusion}
We introduced \model, a model built upon the spectral conjugate split of the shift. We show that the resulting components and induced filters span a novel class of filters with kernels outside ones reachable by the holomorphic class underlying classical spectral convolutions. We illustrate \model's capability to preserve feature distinctiveness with far-reaching receptive fields through examples and the graph transfer task. \model also shows on real benchmarks competitive accuracy for node classification.

\paragraph{Outlooks and limitations.}
Although the transpose is not suitable for a shift-invariant split, is inexpensive, well conditioned, and its Cartesian components have spectra distinct from $\ma L$, which may add expressivity. Performing our splits on both $\ma L$ and $\ma L^T$ could thus improve accuracy at the cost of interpretability. Likewise, the antisymmetric weight parameterizations act on the channel-mixing matrices and are tangent to our operator-level construction. They can nevertheless be combined with our approach. Computationally, the spectral conjugate requires a Schur decomposition, a one-time cost per graph that nonetheless restricts the present implementation to moderate-size graphs. Krylov-Schur~\citep{stewart2002krylov} approximation of the leading modes is a route to scale.


\vfill

\pagebreak



\subsection*{Ethics statement}
No conflict of interest to report nor ethical issues are foreseen in this work.

\subsection*{Acknowledgments}
This work was supported by Swiss National Science Foundation, Sinergia project “Precision mapping of electrical brain network dynamics with application to
epilepsy”, under Grant 209470.

\subsection*{Reproducibility statement}

Here are listed the information to ensure reproducibility of our work:  
\begin{itemize}[leftmargin=5mm, itemsep=0em]
   \item Mathematical concepts are either defined or referenced. Proofs of all Theorems are provided in Appendix~\ref{app:proofs}.
   \item Accompanying information for statements and concepts in the main text are provided in Appendix~\ref{app:accompanying_infos}.
   \item We exactly detail our newly introduced (\model) framework in the main body of our paper Section~\ref{sec:architecture}. Variants of the framework are discussed in Section~\ref{sec:architecture}
   \item The experimental setups for our experiments are described in Appendix~\ref{app:exp_details}.
   \item The datasets used in our experiments are described in Appendix~\ref{app:datasets}.
   \item Hyperparameter settings of the methods used in the experiments are detailed in the Appendix~\ref{app:hyperparam-search-graph-transfer},~\ref{app:hyperparam-search-benchmarking}.
   \item Our code will be made publicly available conditionally on acceptance of the paper.
\end{itemize}


\bibliography{references}
\bibliographystyle{paperstyle}

\appendix
\vfill
\pagebreak
\section{Notations}
\label{app:notations}
\begin{table}[!h]
\centering
\caption{Summary of notations used in the main text.}
\small
\begin{tabular}{ll}
\toprule
\textbf{Notation} & \textbf{Description} \\
\midrule
\multicolumn{2}{l}{\textit{Graph and signals}} \\
$G = (\mathcal{V}, \mathcal{E}, \ma W)$ & Directed weighted graph with nodes $\mathcal{V}$, edges $\mathcal{E}$, weights $\ma W$ \\
$N = |\mathcal{V}|$ & Number of nodes \\
$C$ & Number of feature channels \\
$\bar{\mathcal{E}}$ & Symmetrized edge set, $\mathcal{E} \cup \{(m,n) : (n,m) \in \mathcal{E}\}$ \\
$\ma W \in \mathbb{R}^{N \times N}_{\geq 0}$ & Non-negative weight (adjacency) matrix \\
$\ma D = \operatorname{diag}(\ma W \mathbf{1})$ & Diagonal in-degree matrix \\
$\vc x \in \mathbb{R}^N$ & Graph signal (single channel) \\
$\ma X^{(l)} \in \mathbb{R}^{N \times C}$ & Node feature matrix at layer $l$ \\
$\hat{\vc x}[k]$ & $k$-th graph Fourier coefficient of $\vc x$ \\
$\mathbf{1}$ & All-ones vector (constant signal) \\
\midrule
\multicolumn{2}{l}{\textit{Operators and spectral quantities}} \\
$\ma L = \ma D - \ma W$ & Directed graph Laplacian \\
$\ma L = \ma U \ma\Lambda \ma U^{-1}, \ma L = \ma P \ma J \ma P^{-1}$ & Eigendecomposition and Jordan decomposition of $\ma L$ \\
$(\cdot)^{D}$ & Drazin inverse \\
$\ma P,\ \ma N$ & Generalized eigenvectors and nilpotent part (Jordan form) \\
$\ma U$ & Eigenvector matrix (non-unitary for directed $\ma L$) \\
$\ma\Lambda$ & Diagonal matrix of complex eigenvalues $\lambda_k$ \\
$\lR[n],\ \lI[n]$ & Real and imaginary parts of the $n$-th eigenvalue \\
$\ma L^{\circ}$ & Dissipative operator, $\tfrac{1}{2}(\ma L + \ma L^{\sharp})$ \\
$\ma L^{\uparrow}$ & Asymmetric operator, $\tfrac{1}{2}(\ma L - \ma L^{\sharp})$ \\
$\ma L^{\uparrowcirc}$ & Quotient operator, $\ma L^{\uparrow}(\ma L^{\circ})^{D}$ \\
$\ma L^{\sharp}$ & Spectral conjugate, $\ma P \overline{\ma\Lambda} \ma P^{-1}$ \\
$\sigma(\ma L)$ & Spectrum of $\ma L$ \\
$\sigma_R(\ma L)$ & Real axis spectrum of $\ma L$ \\
$\sigma_I(\ma L)$ & Imaginary axis spectrum of $\ma L$ \\
$\mathcal{S}(\cdot)$ & Smoothness measure\\
\midrule
\multicolumn{2}{l}{\textit{Kernels, filters and reparameterization}} \\
$\tau : \mathbb{C} \to \mathbb{C}$ & Conjugate kernel \\
$h : \mathcal{D} \subset \mathbb{C} \to \mathbb{C}$ & Spectral kernel \\
$h_{\sumf},\ h_{\ratf}$ & Sum and \filtername kernels \\
$K,\ K_1,\ K_2$ & Polynomial degrees \\
$(\vc a[k],\ \vc b[k])$, $(\vct\alpha[k],\ \vct\beta[k])$ & Sum filter coefficients (canonical and reparameterized) \\
$\vc c[k], \vct\eta[k]$ & Ratio filter coefficients (canonical and reparameterized) \\
$\vct\theta[k]$ & Standard polynomial filter coefficients \\
$T_k(\cdot)$ & Chebyshev polynomial of the first kind, degree $k$ \\
$\widetilde{\ma L}^{\circ},\ \widetilde{\ma L}^{\uparrow}$ & Spectrally rescaled $\Lc$ and $\Lu$ \\
$\ma Z = -j\,\ma L^{\uparrowcirc},\ \widetilde{\ma Z}$ & Real-spectrum quotient operator and its rescaling \\
$c_R,\ c_Q$ & Spectral centering constants \\
$R_R,\ R_I,\ R_Q$ & Spectral scaling radii \\
\midrule
\multicolumn{2}{l}{\textit{Architecture and analysis}} \\
$\ma\Theta_k^{(l,\circ)},\ \ma\Theta_k^{(l,\uparrow)}$ & Learnable weight matrices at layer $l$ \\
$\ma B^{(l)}$ & Learnable bias at layer $l$ \\
$\nu,\rho$ & Learnable parameter shared across layers \\
$\psi(\cdot)$ & Pointwise non-linear activation \\
$\mu_k,\ z_k$ & Argument-ordering and shift eigenvalues on $\mathcal{C}_N$ \\
$\mathcal{C}_N$ & Directed cycle graph on $N$ nodes \\
\midrule
\multicolumn{2}{l}{\textit{Complex analysis}} \\
$\Gamma,\ \Omega$ & Integration contour and enclosed region \\
$\bar{\partial} = \partial/\partial\bar{z}$ & Wirtinger derivative (antiholomorphic)\\
$\partial = \partial/\partial z$ & Wirtinger derivative (holomorphic)\\
$\overline{\mathcal{D}}$ & Closure of domain of definition $\mathcal{D}$ \\
$j$ & Imaginary unit, $j^2 = -1$ \\
\bottomrule
\end{tabular}
\end{table}

\pagebreak
\section{Proofs of Theorems and Propositions}
\label{app:proofs}

\subsection{Proof of Proposition~\ref{prop:cartesian-commute} (Cartesian split breaks shift invariance)}
\label{proof:cartesian-commute}

\begin{lemma}
Denoting the commutator $\bigl[\ma A, \ma B\bigr]=\ma A\ma B - \ma B\ma A$, we have that for any $\ma L\in\mathbb{R}^{N\times N}$,
\begin{equation}
    \bigl[\ma L,\;\tfrac{1}{2}(\ma L\pm\ma L^{T})\bigr]
    \;=\;\pm\tfrac{1}{2}\bigl[\ma L,\;\ma L^{T}\bigr].
    \label{eq:commutator-identity}
\end{equation}
\end{lemma}
\begin{proof}
For any $\ma L\in\mathbb{R}^{N\times N}$, 
\begin{equation}
  \bigl[\ma L,\;\tfrac{1}{2}(\ma L\pm\ma L^{T})\bigr] = \tfrac{1}{2}\left(\ma L(\ma L\pm \ma L^{T}) - (\ma L\pm \ma L^{T})\ma L\right) = \pm\tfrac{1}{2}(\ma L\ma L^{T} - \ma L^{T}\ma L)= \pm\tfrac{1}{2}\bigl[\ma L,\;\ma L^{T}\bigr].
\end{equation}    
\end{proof}
If $\ma L$ is normal, then $\ma L\ma L^{T} = \ma L^{T}\ma L$ and hence $\bigl[\ma L,\;\tfrac{1}{2}(\ma L\pm\ma L^{T})\bigr]=0$, so the Cartesian components commute with $\ma L$. Conversely, if the Cartesian components commute with $\ma L$, then $[\ma L,\;\ma L^T]=0$ and $\ma L$ is normal. Hence the Cartesian components commute with $\ma L$ if and only if $\ma L$ is normal. \qed

\subsection{Proof of Theorem~\ref{thm:conjugate-properties} (Conjugate properties)}
\paragraph{Proof of main theorem.} We refer to \citet{nevanlinna2018non}[Proposition 5.1] for the commutativity of $\ma L$ and $\ma L^{\sharp}$. If $\ma L$ is normal, then $\ma L^{\sharp} = \ma U \overline{\ma\Lambda} \ma U^{H} = \ma L^{T}$. Conversely, if $\ma L^{\sharp} = \ma L^{T}$, then as $\ma L^{\sharp}$ commutes with $\ma L$, $\ma L$ is normal. 

What remains is the proof of the realness of $\ma L^{\sharp}$ when $\ma L$ is real.
\begin{lemma}\cite[Eq.~1.9 -- Block action]{nevanlinna2018non}
\label{lem:block-action}
Let $h\in\mathcal{C}^{N-1}(\overline{\mathcal{D}})$ and let $\ma J_n=\lambda_n\ma I+\ma N_n$ be a Jordan block of size $m$. Then
\begin{equation}
    h(\ma J_n)=\sum_{k=0}^{m-1}\frac{(\partial^{k}h)(\lambda_n)}{k!}\,\ma N_n^{k},
    \qquad \partial=\frac{\partial}{\partial z},
    \label{eq:block-wirtinger}
\end{equation}
where $\partial$ denotes the Wirtinger derivative instead of the complex derivative. This is a generalization of~\eqref{eq:jordan-filter} to non-holomorphic functions.
\end{lemma}
Only $\partial$ appears in \eqref{eq:block-wirtinger} and given that the conjugate kernel $\tau(z)=\bar z$ satisfies $\partial\tau\equiv0$, so every term with $k\geq1$ vanishes and
\begin{equation}
    \tau(\ma J_n)=\overline{\lambda_n}\,\ma I,
    \label{eq:tau-block}
\end{equation}
as defined in Definition~\ref{def:spectral-conjugate}.

\begin{lemma}[Conjugation]
\label{lem:conjugation}
Let $h\in\mathcal{C}^{N-1}(\overline{\mathcal{D}})$. Setting $h^{*}(z):=\overline{h(\bar z)}$, we have $h^{*}\in\mathcal{C}^{N-1}(\overline{\mathcal{D}})$ and
\begin{equation}
    \overline{h(\ma A)}=h^{*}(\overline{\ma A})
    \qquad\text{for every }\ma A\in\mathbb{C}^{N\times N}.
    \label{eq:conjugation}
\end{equation}
\end{lemma}
\begin{proof}
Conjugation exchanges the Wirtinger derivatives, $\partial\,\overline{F}=\overline{\bar\partial F}$, so $(\partial h^{*})(z)=\overline{(\partial h)(\bar z)}$ and, inductively, $(\partial^{k}h^{*})(\bar\lambda)=\overline{(\partial^{k}h)(\lambda)}$. If $\ma A=\ma P\ma J\ma P^{-1}$ then $\overline{\ma A}=\overline{\ma P}\,\overline{\ma J}\,\overline{\ma P}^{-1}$, and $\overline{\ma J}$ is a Jordan form whose blocks are $\bar\lambda_n\ma I+\ma N_n$, the nilpotent parts being real. Applying \eqref{eq:block-wirtinger} blockwise,
\begin{equation}
    \overline{h(\ma J_n)}
    =\sum_{k=0}^{m-1}\frac{\overline{(\partial^{k}h)(\lambda_n)}}{k!}\,\ma N_n^{k}
    =\sum_{k=0}^{m-1}\frac{(\partial^{k}h^{*})(\overline{\lambda_n})}{k!}\,\ma N_n^{k}
    =h^{*}(\overline{\ma J_n}),
\end{equation}
and conjugating $\ma P\,h(\ma J)\,\ma P^{-1}$ gives \eqref{eq:conjugation}.
\end{proof}
Applying Lemma~\ref{lem:conjugation} to the conjugate kernel: $\tau^{*}(z)=\overline{\tau(\bar z)}=\overline{\overline{\bar z}}=\bar z=\tau(z)$. Since $\ma L$ is real, $\overline{\ma L}=\ma L$, and \eqref{eq:conjugation} gives
\begin{equation}
    \overline{\ma L^{\sharp}}=\overline{\tau(\ma L)}=\tau^{*}(\overline{\ma L})=\tau(\ma L)=\ma L^{\sharp}.
\end{equation}
A matrix equal to its own conjugate is real.
\qed

\paragraph{Proof of corollary.} The decomposition $\ma L=\tfrac{1}{2}\bigl(\ma L+\ma L^{\sharp}\bigr) +\tfrac{1}{2}\bigl(\ma L-\ma L^{\sharp}\bigr)$ is exact. Both components commute with $\ma L$ and with each other since $\ma L^{\sharp}$ commutes with $\ma L$. As $\ma L^{\sharp}$ is real, both components of the decomposition are also real. Also by definition, $\sigma(\ma L^{\circ})=\sigma_R(\ma L)$, $\sigma(\ma L^{\uparrow})=\sigma_I(\ma L)$. Following, when $\ma L$ is normal, this decomposition reduces to the Cartesian split. \qed

\subsection{Proof of Theorem~\ref{thm:non-holomorphicity} (Vertex-domain sum and \filtername filters)}
\paragraph{Non-holomorphicity.} Decomposing $\Re(z) = (z+\bar{z})/2$ and $\Im(z) = (z-\bar{z})/(2j)$, the Wirtinger derivative $\bar{\partial}h_\sumf, \bar{\partial}h_\ratf$ simplifies to:
\begin{equation}
    \bar{\partial} h_\sumf(z) = \sum_{k=1}^{K_1} \frac{k\, \vc a[k]}{2^k} (z+\bar{z})^{k-1} + \sum_{k=1}^{K_2}  \frac{-k\, \vc b[k]}{(2j)^k} (z-\bar{z})^{k-1} \neq 0,
\end{equation}
\begin{equation}
  \bar{\partial} h_\ratf(z) = \sum_{k=1}^{K} \vc c[k] \frac{-2k z (z-\bar{z})^{k-1}}{j^k(z + \bar{z})^{k+1}} \neq 0,
\end{equation}
for non-zero $\vc a, \vc b, \vc c$. This shows that both kernels are not holomorphic, as their Wirtinger derivative is non-zero.

\paragraph{Vertex-domain expressions.}
Both filters follow from the functional calculus of \citet{nevanlinna2018non} applied to kernels built from the identity and the conjugation $\tau(z)=\bar z$.
\begin{lemma}[{\cite[Cor.~3.5]{nevanlinna2018non}}]
\label{lem:homomorphism}
Let $\ma L\in\C^{N\times N}$ and $\mathcal{D}\subset\C$ be open with $\mathcal{D}\supset\sigma(\ma L)$. The map $h\mapsto h(\ma L)$ from $\mathcal{C}^{N-1}(\overline{\mathcal{D}})$ to $\C^{N\times N}$ is a homomorphism.
\end{lemma}
Write $\kappa^{\circ}=\tfrac{1}{2}(\mathrm{id}+\tau)$ and $\kappa^{\uparrow}=\tfrac{1}{2}(\mathrm{id}-\tau)$, so that linearity of Lemma~\ref{lem:homomorphism} together with $\tau(\ma L)=\ma L^{\sharp}$ gives $\kappa^{\circ}(\ma L)=\Lc$ and $\kappa^{\uparrow}(\ma L)=\Lu$, and multiplicativity gives $(\kappa^{\circ})^{k}(\ma L)=(\Lc)^{k}$ and $(\kappa^{\uparrow})^{k}(\ma L)=(\Lu)^{k}$ for every $k\in\mathbb{N}$. Since $h_{\sumf}=\sum_{k=0}^{K_1}\vc a[k](\kappa^{\circ})^{k}+\sum_{k=0}^{K_2}\vc b[k](\kappa^{\uparrow})^{k}$, linearity yields \eqref{eq:graph-sum-filter} directly:
\begin{equation}
  h_{\sumf}(\ma L)=\sum_{k=0}^{K_1}\vc a[k]\,(\Lc)^{k}+\sum_{k=0}^{K_2}\vc b[k]\,(\Lu)^{k}.
\end{equation}
The \filtername kernel requires more care, since it involves a quotient and is discontinuous where $\Re(z)=0$. We assume throughout that $\lambda=0$ is the only eigenvalue of $\ma L$ on the imaginary axis, which holds for the Laplacian of any connected graph. As $\sigma(\ma L)$ is finite, fix $\varepsilon>0$ small enough that the discs $B(\lambda_n,\varepsilon)$ are pairwise disjoint and $\kappa^{\circ}$ does not vanish on any $B(\lambda_n,\varepsilon)$ with $\lambda_n\neq0$, and set $\mathcal{D}=\bigcup_n B(\lambda_n,\varepsilon)$. Being open and disconnected, $\mathcal{D}$ admits kernels prescribed independently on each set, which lets us define
\begin{equation}
    \xi(z)=
    \begin{cases}
      1/\kappa^{\circ}(z), & z\in B(\lambda_n,\varepsilon),\ \lambda_n\neq0,\\[2pt]
      0, & z\in B(0,\varepsilon).
    \end{cases}
    \label{eq:drazin-kernel}
\end{equation}
On each set $\xi$ is either constant or the reciprocal of a non-vanishing polynomial in $(z,\bar z)$, so $\xi\in\mathcal{C}^{N-1}(\overline{\mathcal{D}})$ and Lemma~\ref{lem:homomorphism} applies. By construction $\xi$ inverts the nonzero part of $\sigma(\Lc)$ and sets the rest to $0$, which is the spectral characterisation of the Drazin inverse~\citep{campbell2009generalized}, thence $\xi(\ma L)=(\Lc)^{D}$. Multiplicativity then gives
\begin{equation}
    (\kappa^{\uparrow}\xi)(\ma L)=\kappa^{\uparrow}(\ma L)\,\xi(\ma L)=\Lu(\Lc)^{D}=\ma L^{\uparrowcirc},
\end{equation}
and, applied repeatedly, $(\kappa^{\uparrow}\xi)^{k}(\ma L)=(\ma L^{\uparrowcirc})^{k}$ for all $k\geq0$. Finally $h_{\ratf}=\vc c[0]+\sum_{k=1}^{K}\vc c[k](\kappa^{\uparrow}\xi)^{k}$ on $\mathcal{D}$, where on the points with $\lambda_n\neq0$ we have $\kappa^{\uparrow}\xi=\kappa^{\uparrow}/\kappa^{\circ}$, recovering the quotient of \eqref{eq:kernel-rational}, while on $B(0,\varepsilon)$ every term with $k\geq1$ vanishes and the kernel reduces to $\vc c[0]$, matching $h_{\ratf}$ at $\Re(z)=0$. Linearity then gives
\begin{equation}
    h_{\ratf}(\ma L)=\sum_{k=0}^{K}\vc c[k]\,\bigl(\ma L^{\uparrowcirc}\bigr)^{k}.
\end{equation}
\qed

\subsection{Proof of Theorem~\ref{thm:real_matrix} (Real-valued twins filters)}
As in~\citet{chan2026graph}, leveraging the Vandermonde matrix, any real-valued filters expressible by $h_{\sumf}(\ma L)$ and $h_{\ratf}(\ma L)$ i.e. elements of respectively:
\begin{equation}
    \{h_{\sumf}(\ma L): (\vc a, \vc b)\in \C^{K_1+1}\times\C^{K_2+1}\}\cap\R^{N\times N} \quad \text{and} \quad \{h_{\ratf}(\ma L): \vc c\in \C^{K+1}\}\cap\R^{N\times N} 
\end{equation}
as defined in~\eqref{eq:graph-sum-filter}, can be expressed as elements of $\{h_\sumf(\ma L): (\vc a, \vc b)\in\R^{K_1+1}\times\R^{K_2+1}\}$ and $\{h_\ratf(\ma L): \vc c\in\R^{K+1}\}$ guaranteeing their equality --- this holds for both diagonalizable and non-diagonalizable $\ma L$ thanks to conjugate eigenvalues occurrence in the Jordan decomposition. Since we need $h_{\sumf}(\ma L), h_{\ratf}(\ma L)$ to be real, we start with $\vc a, \vc b, \vc c\in\R$. By identifying the unique coefficients for the Chebyshev expansions of $\wLc, \wLu, \widetilde{\ma Z}$, starting with maximal order $K$
\begin{align}
\vct \alpha[K] T_K(\wLc) &= \vct \alpha[K] \sum_{k=0}^{K} \vc a_K[k] (\wLc)^k = \vct \alpha[K] \sum_{k=0}^{K} \left(\frac{\vc a_K[k]}{R_R} \sum_{m=0}^k \binom{k}{m}(\Lc)^m(c_R \ma I)^{k-m}\right),\\
\vct \beta[K] T_K(\wLu) &= \vct \beta[K] \sum_{k=0}^{K} \vc b_K[k] (\wLu)^k = \vct \beta[K] \sum_{k=0}^{K} (-j)^k\frac{\vc b_K[k]}{R_I} (\Lu)^k,\\
\vct \eta[K] T_K(\widetilde{\ma Z}) &= \vct \eta[K] \sum_{k=0}^{K} \vc c_K[k] (\widetilde{\ma Z})^k = \vct \eta[K] \sum_{k=0}^{K}(-j)^k \frac{\vc c_K[k]}{R_Q}(\ma L^{\uparrowcirc})^k.
\end{align}
Given $T_K$ is of maximal order $K$, no other terms except for $\vct \alpha[K]$ contribute to the expansion in $\vc a[K]$, therefore we have exactly $\vct \alpha[K]\frac{\vc a_K[K]}{R_R} = \vc a[K]$. Similarly, considering the two remaining expressions $\vct \beta[K](-j)^K\frac{\vc b_K[K]}{R_I} = \vc b[K]$, and $\vct \eta[K](-j)^K\frac{\vc c_K[K]}{R_Q} = \vc c[K]$. Hence the following holds:
\begin{itemize}[leftmargin=10mm, itemsep=0em]
  \item $\vct \alpha[K]$ is real,
  \item $\vct \beta[K]$ is real for even $K$ and purely imaginary for odd $K$,
  \item $\vct \eta[K]$ is real for even $K$ and purely imaginary for odd $K$.
\end{itemize}
For lower order coefficients, we can apply the same argument recursively (while subtracting the previous orders for which we have already determined the coefficients) to show that $\vct \alpha[k]$ is real for all $k$ and $\vct \beta[k], \vct \eta[k]$ are real for even $k$ and purely imaginary for odd $k$. As this procedure can be performed for all $\vc a,\vc b, \vc c\in\R$,
\begin{align}
&\{h_{\sumf}(\ma L) : (\vct\alpha,\vct\beta)\in\C^{K_1+1}\times\C^{K_2+1}\}\cap\R^{N\times N} \nonumber\\
&\qquad= \{h_{\sumf}(\ma L) : \vct\alpha[k]\in\R,\ \vct\beta[2k]\in\R,\ \vct\beta[2k+1]\in j\R\},\\[4pt]
&\{h_{\ratf}(\ma L) : \vct\eta\in\C^{K+1}\}\cap\R^{N\times N} \nonumber\\
&\qquad= \{h_{\ratf}(\ma L) : \vct\eta[2k]\in\R,\ \vct\eta[2k+1]\in j\R\}.
\end{align}
This completes the proof of Theorem~\ref{thm:real_matrix}. \qed

\paragraph{Further intuition on reparameterized matrices.}
Additional intuition on the nature of the real-matrix condition is provided in the following lemma and expressions, which in one hand show how the Chebyshev polynomial parity interacts with the matrix parity of $\wLu$ and $\widetilde{\ma Z}$ to yield a real-matrix filter and in the other hand show how the spectral response is decomposed into real and imaginary parts that are even and odd in $\lI$ respectively.
\begin{lemma}[Parity of Chebyshev polynomials]
\label{lem:parity-cheb}
Let $x\in[-1,1]$, the k-th degree Chebyshev polynomials $T_k$ satisfies 
\begin{equation}
    T_k(-x)=(-1)^kT_k(x)\Rightarrow
    \begin{cases}
    \text{even function} & k\text{ even,}\\ \text{odd function} & k\text{ odd.}
    \end{cases}
\end{equation}
\end{lemma}
\begin{proof}
For $x\in[-1,1]$, $T_k(-x)=\cos(k\arccos(-x))=\cos(k(\pi-\arccos(x))) = (-1)^kT_k(x)$.
\end{proof}
\begin{lemma}[Matrix parity of Chebyshev polynomials]
\label{lem:parity}
Let $\ma Q$ be a matrix with purely imaginary entries, giving
\begin{equation}
  \ma Q^i \;\in\;
  \begin{cases}\R^{N\times N} & i\text{ even,}\\ j\R^{N\times N} & i\text{ odd.}\end{cases}
  \label{eq:Qpow}
\end{equation}
Parallely, from Lemma~\ref{lem:parity-cheb} the Chebyshev polynomial $T_k$ has the same parity as $k$, so it contains only monomials $x^i$ with $i\equiv k\pmod{2}$. Combining with~\eqref{eq:Qpow}:
\begin{equation}
  A_k = T_k(\ma Q) \;\in\;
  \begin{cases}
    \R^{N\times N}   & k\text{ even,}\\
    j\,\R^{N\times N}& k\text{ odd.}
  \end{cases}
  \label{eq:Ak_parity}
\end{equation}
\end{lemma}
\begin{proof}
Write $T_k(x)=\sum_{i\equiv k\,(2)}a_ix^i$ with $a_i\in\R$.  Then $A_k=\sum_{i\equiv k\,(2)}a_i\ma Q^i$. Each summand $a_i\ma Q^i$ lies in $\R^{N\times N}$ if $i$ even, in $j\R^{N\times N}$ if $i$ odd, and all summands share the same parity as $k$, so the whole sum is either in $\R^{N\times N}$ or $j\R^{N\times N}$.
\end{proof}

\paragraph{Frequency response as a function of Chebyshev polynomials.} 
\label{app:freq-response-cheb}
The spectral response at eigenvalue $\lambda[n]=\lR[n]+j\lI[n]$ of the sum filter $h_{\sumf}(\lambda[n])$ can be expressed in terms of Chebyshev polynomials as
\begin{align}
  \Re\bigl(h_{\sumf}(\lambda[n])\bigr)
  &= \sum_k\vct \alpha[k] T_k(x_{R,n})
   + \sum_i\vct \beta[2i]\,T_{2i}(x_{I,n}),
  \label{eq:re_part}\\
  \Im\bigl(h_{\sumf}(\lambda[n])\bigr)
  &= \sum_i\vct \beta[2i+1]\,T_{2i+1}(x_{I,n}),
  \label{eq:im_part}
\end{align}
where $x_{R,n}=(\lR[n]-c_R)/R_R$ and $x_{I,n}=\lI[n]/R_I$. For the \filtername filter $h_{\ratf}(\lambda[n])$ we have
\begin{align}
  \Re\bigl(h_{\ratf}(\lambda[n])\bigr)
  &= \sum_i\vct \eta[2i]\,T_{2i}(x_{Q,n}),
  \label{eq:re_part_rat}\\
  \Im\bigl(h_{\ratf}(\lambda[n])\bigr) &= \sum_i\vct \eta[2i+1]\,T_{2i+1}(x_{Q,n}),
  \label{eq:im_part_rat}
\end{align}
where $x_{Q,n}=\left(-j\tfrac{j\lI[n]}{\lR[n]}\right)/R_Q$.

\subsection{Proof of Theorem~\ref{thm:fir-iir} (IIR correspondence)}

\paragraph{M\"obius identity.}
On the cycle, $\lR[n]=1-\cos\vct \theta[n]$ and $\lI[n]=-\sin\vct \theta[n]$, so
\begin{equation} \label{eq:mobius}
\vct \mu[n]=\frac{j\lI[n]}{\lR[n]} =\frac{-j\sin\vct \theta[n]}{1-\cos\vct \theta[n]} = \frac{-2j \sin(\vct \theta[n]/2)\cos(\vct \theta[n]/2)}{2\sin^2(\vct \theta[n]/2)} = -j\cot\!\left(\frac{\vct \theta[n]}{2}\right).
\end{equation}
Writing $\vct z[k]=e^{j\vct \theta[n]}$ and using $\cot(\theta/2)=j(z+1)/(z-1)$,
\begin{equation}
\vct \mu[n]=-j\cdot j\,\frac{\vct z[n]+1}{\vct z[n]-1}=\frac{\vct z[n]+1}{\vct z[n]-1},
\end{equation}
which is \eqref{eq:mobius}. This is a M\"obius transform, projecting from the unit circle onto the imaginary axis.

\paragraph{Rational structure \eqref{eq:rational}.}
Substituting \eqref{eq:mobius} into a degree-$K$ polynomial,
\begin{equation}
\sum_{k=0}^{K}\vc c[k]\vct \mu[n]^k 
=\frac{\sum_{k=0}^{K}\vc c[k](\vct z[n]+1)^k(\vct z[n]-1)^{K-k}}{(\vct z[n]-1)^{K}}
=\frac{Q_K(\vct z[n])}{(\vct z[n]-1)^{K}},
\end{equation}
where $Q_K$ is a polynomial of degree $\le K$ obtained by multiplying the denominator. This is a rational function of $\vct z[n]$ of type $[K/K]$ whose denominator $(\vct z[n]-1)^K$ places all $K$ poles at $\vct z=1$, the DC mode. A rational spectral response corresponds to an IIR (ARMA) vertex-domain filter and hence has non-compact support. \qed

\subsection{Proof of Theorem~\ref{thm:reach-gap} (Reachability gap)}

\paragraph{Starting notation.}
$\Delta u[\ell]=u[\ell+1]-2u[\ell]+u[\ell-1]$ is the discrete Laplacian and $\vc{g}_k:=(\Lrat)^k\vc{\delta}_0$ the $k$-th kernel, so that $\vc{g}_k[d]$ is the impulse response of $(\Lrat)^k$ at distance $d$. Since $\ma{L}$ is normal on $C_N$, the spectral split gives
\begin{equation}
\Lc=\ma{I}-\tfrac12(\ma{W}+\ma{W}^T), \qquad \Lu=\tfrac12(\ma{W}^T-\ma{W}),
\label{eq:split}
\end{equation}
so that, acting on a signal $u\in\R^N$,
\begin{equation}
(\Lc u)[\ell]=-\tfrac12(\Delta u)[\ell], \qquad (\Lu u)[\ell]=\tfrac12\bigl(u[\ell+1]-u[\ell-1]\bigr):
\label{eq:stencils}
\end{equation}
a second difference and a centered first difference. 

\begin{prop}[Identity on \filtername operator]
\label{prop:poisson}
$\Lc\Lrat=\Lu$, and $\Lrat\vc{1}=\vc{0}$, $\vc{1}^T\Lrat=\vc{0}^T$.
\end{prop}

\begin{proof}
All operators commute, so it suffices to observe their spectral definition.
\end{proof}

\begin{lemma}[Linear form of $\vc g_1$]
\label{lem:ramp}
$\vc{g}_1[\ell]=2\ell/N-1$ for $1\le\ell\le N-1$, and $\vc{g}_1[0]=0$.
\end{lemma}
\begin{proof}Apply Proposition~\ref{prop:poisson} to $\vc{\delta}_0$ and use \eqref{eq:stencils}:
\begin{equation}
\Lc\vc{g}_1=\Lu\vc{\delta}_0=\tfrac12(\vc{\delta}_{-1}-\vc{\delta}_{1})
\qquad\Longleftrightarrow\qquad
\Delta\vc{g}_1=\vc{\delta}_1-\vc{\delta}_{-1}.
\label{eq:dipole}
\end{equation}
By \eqref{eq:dipole} the second difference of $\vc{g}_1$ vanishes at every vertex other than $\ell=\pm1$. A sequence with vanishing second difference on a discrete interval is affine on that interval, so $\vc{g}_1[\ell]=a+b\ell$ for $2\le\ell\le N-2$. Evaluating $\Delta \vc g_1$ at $\ell=0$ and inserting the affine form gives $\vc g_1[0]=a + \frac{bN}{2}$. An additional condition comes from the fact that $\Lrat$ sets the zero-th frequency to $0$, i.e., $\hat{\vc{g}}_1[0]=\mu[0]=0$ and therefore $\sum_{\ell}\vc{g}_1[\ell]=0$. Combining these two conditions gives $a=-\frac{bN}{2}$ and $\vc g_1[0] = 0$. Once again evaluating $\Delta \vc g_1$ at $\ell=1$ gives $\vc g_1[2]-2\vc g_1[1]+\vc g_1[0]=1$, which with the affine form gives $b=2/N$. Hence $\vc{g}_1[\ell]=2\ell/N-1$ for $1\le\ell\le N-1$.
\end{proof}

\begin{lemma}[Degree and leading coefficient]
\label{lem:degree}
For $N\ge3k+2$ there is a polynomial $P_k$ of degree exactly $k$ with $\vc{g}_k[d]=P_k(d)$ on the interval $k<d<N-k$, whose leading coefficient obeys
\begin{equation}
a_k=-\frac{2a_{k-1}}{k},\qquad a_1=\frac2N, \qquad\Longrightarrow\qquad a_k=\frac{(-1)^{k-1}2^{k}}{N\,k!}. \label{eq:leadrec}
\end{equation}
\end{lemma}

\begin{proof}
We proceed by induction on $k$,

Base case $k=1$: Lemma~\ref{lem:ramp} gives $P_1(d)=2d/N-1$, of degree $1$
with $a_1=2/N$.

Inductive hypothesis: Suppose the statement holds at order $k-1$, so that $\vc{g}_{k-1}=P_{k-1}$ on $k-1<d<N-k+1$ with $\deg P_{k-1}=k-1$ and leading coefficient $a_{k-1}\neq0$.

Applying Proposition~\ref{prop:poisson} to $\vc{g}_k=\Lrat\vc{g}_{k-1}$ and using \eqref{eq:stencils},
\begin{equation}
-\tfrac12\,\Delta\vc{g}_k=\Lu\vc{g}_{k-1}.
\label{eq:krec}
\end{equation}
Let $d$ lie in the interval $k<d<N-k$. Then both $d-1$ and $d+1$ lie in $k-1<d<N-k+1$, where the inductive hypothesis applies, so by \eqref{eq:stencils} the right-hand side of \eqref{eq:krec} equals
\begin{equation}
Q(d):=\tfrac12\bigl(P_{k-1}(d+1)-P_{k-1}(d-1)\bigr).
\end{equation}
Expanding, the terms of degree $k-1$ cancel and $Q$ has degree exactly $k-2$ with leading coefficient $(k-1)a_{k-1}$.

On the interval, \eqref{eq:krec} reads $\Delta\vc{g}_k=-2Q$ and to it we seek a particular solution that is a polynomial of degree $k$ with leading coefficient $b$. Since
$\Delta(b\,d^{k})=b\,k(k-1)\,d^{k-2}+O(d^{k-3})$, matching the coefficients of $d^{k-2}$ gives
\begin{equation}
b\,k(k-1)=-2(k-1)a_{k-1} \qquad\Longrightarrow\qquad b=-\frac{2a_{k-1}}{k}\neq0 .
\end{equation}
The homogeneous solutions of $\Delta u=0$ on a discrete interval are the affine sequences, which for $k\ge2$ cannot cancel the term of degree $k$, so $a_k\neq 0$. Hence $\vc{g}_k$ agrees on the interval with a polynomial of degree exactly $k$, and $a_k=b$, which is \eqref{eq:leadrec}. The hypothesis $N\ge3k+2$ guarantees that the interval contains more than $k+1$ vertices, so the polynomial is not degenerate. Unrolling the recursion leads to \eqref{eq:leadrec}.
\end{proof}

\begin{lemma}[Bernoulli form]
\label{lem:bernoulli}
Let $s=d/N$ and $c_k=(-1)^{k-1}2^{k}/k!$. Then, uniformly on $k<d<N-k$,
\begin{equation}
\vc{g}_k[d]=N^{k-1}c_k\,B_k(s)+O\!\left(N^{k-3}\right),
\label{eq:bform}
\end{equation}
with equality for $k=1$.
\end{lemma}

\begin{proof}
Set $\tilde P_k(s):=P_k(Ns)/N^{k-1}$, the $k$-th kernel expressed in the fraction distance $s=d/N$. Its leading term is $a_k(Ns)^{k}/N^{k-1}=N a_k s^{k}$, so by \eqref{eq:leadrec} its leading coefficient is $Na_k=(-1)^{k-1}2^{k}/k!=:c_k$, which is independent of $N$. This normalization is what allows a limit to be taken when computing derivatives.

Under $d=Ns$, the second difference acts as $N^{-2}\,\mathrm{d}^2/\mathrm{d}s^2$ and the centered first difference as $N^{-1}\,\mathrm{d}/\mathrm{d}s$, each with relative error $O(N^{-2})$ on polynomials of fixed degree. Hence differentiating over $s$, \eqref{eq:krec} becomes $P_k''=-2N\,P_{k-1}'$, and substituting $P_k=N^{k-1}\tilde P_k$ cancels every power of $N$, leaving
\begin{equation}
\tilde P_k''=-2\,\tilde P_{k-1}',
\qquad
\tilde P_1(s)=2s-1=2B_1(s).
\label{eq:contrec}
\end{equation}

Try $\tilde P_k=c_kB_k$. The Bernoulli polynomials satisfy $B_k'=k\,B_{k-1}$, hence $B_k''=k(k-1)B_{k-2}$, so \eqref{eq:contrec} holds if and only if
\begin{equation}
c_k\,k(k-1)=-2\,c_{k-1}(k-1) \qquad\Longleftrightarrow\qquad c_k=-\frac{2c_{k-1}}{k},
\end{equation}
which is the recursion \eqref{eq:leadrec} of Lemma~\ref{lem:degree}. With $c_1=2$ from \eqref{eq:contrec}, this gives $c_k=(-1)^{k-1}2^{k}/k!$.

To fully determine solution of the second-order~\eqref{eq:contrec} it first satisfies, $\int_0^1\tilde P_k=0$, which follows from $\hat{\vc{g}}_k[0]=\mu[0]^k=0$. The second is $\tilde P_k(1-s)=(-1)^k\tilde P_k(s)$, which follows from $\vc{g}_k[N-\ell]=(-1)^k\vc{g}_k[\ell]$. This equality is surmised from the expression in Lemma~\ref{lem:ramp} where $\vc g_1$ is an odd sequence, furthermore $\vc{g}_k=\vc{g}_1^{*k}$, and the convolution of two odd sequences is even, so parity alternates with $k$. Both conditions are satisfied for $B_k$, by its vanishing mean and by the reflection formula $B_k(1-s)=(-1)^kB_k(s)$. The ansatz of Bernoulli polynomials therefore holds.

The only approximation is the replacement of differences by derivatives when considering polynomials, leading to $O(N^{-2})$ term. At $k$ steps relative error remains $O(N^{-2})$, hence the recursion gives \eqref{eq:bform}. For $k=1$ no approximation is made, since Lemma~\ref{lem:ramp} is exact.
\end{proof}

\paragraph{Proof of Theorem~\ref{thm:reach-gap}.}

(i): Since $\ma{W}^m\vc{\delta}_0=\vc{\delta}_m$, the kernel of $h=\sum_{m\le K}\gamma_m\ma{W}^m$ is $\sum_{m\le K}\gamma_m\vc{\delta}_m$. It therefore equals $\gamma_d$ at $d\le K$ and vanishes identically for $d>K$, hence $S_h(d)=0$ there. The coefficients $\gamma_m$ are unconstrained, so the response on the first $K$ hops is arbitrary.

(ii): By \eqref{eq:split}, $e^{-t\Lc}=e^{-t}\exp\bigl(\tfrac{t}{2}(\ma{W}+\ma{W}^T)\bigr)$. The generating identity $e^{\frac{t}{2}(z+z^{-1})}=\sum_{d}I_d(t)z^{d}$, where $I_d$ is the modified Bessel function of first kind \citep{watson1944}, leads to the impules response $h[d,0]$ as $e^{-t}I_d(t)$, and the peak is at $d=0$ because $I_0\ge I_d$ for all $d$. For the bound of $S_h(d)=I_d(t)/I_0(t)$, compare the defining series termwise: the $m$-th term of $I_d(t)$ is $(t/2)^{d+2m}/\bigl(m!\,(d+m)!\bigr)$ and that of $\tfrac{(t/2)^d}{d!}I_0(t)$ is $(t/2)^{d+2m}/\bigl(d!\,(m!)^2\bigr)$, so their ratio is $d!\,m!/(d+m)!\le1$. Summing over $m$ gives $I_d(t)\le\tfrac{(t/2)^d}{d!}I_0(t)$, and $S_h(d)\le \tfrac{(t/2)^d}{d!}$, which is the claim.

(iii): The term $\vct \eta[0]\ma{I}_N$ has kernel $\vct \eta[0]\vc{\delta}_0$, which vanishes on $K<d<N-K$; only the terms with $k\ge1$ contribute there. On the same interval $K<d<N-K$, Lemma~\ref{lem:bernoulli} applies to each $\vc{g}_k$ with $k\le K$. A weighted sum $\vc g_k$ with weights $\vct \eta[k]$ gives $p(s)$ in \eqref{eq:ratsens}. Substituting $p(s) \approx h[d,0]$ into $S_h(d)$, first yields $|p(s)|/\max_{q\in[0,1]}|p(q)|$. The remaining discrepancy is between the maximum over the vertices $v \in\mathcal{V}$ and the supremum over the continuum $[0,1]$, which for polynomials of fixed degree is $O(N^{-2})$.

\begin{figure}[H]
    \centering
    \captionsetup[subfigure]{justification=centering}
  \subfloat[Sensitivity profiles]{%
       \includegraphics[width=0.5\linewidth]{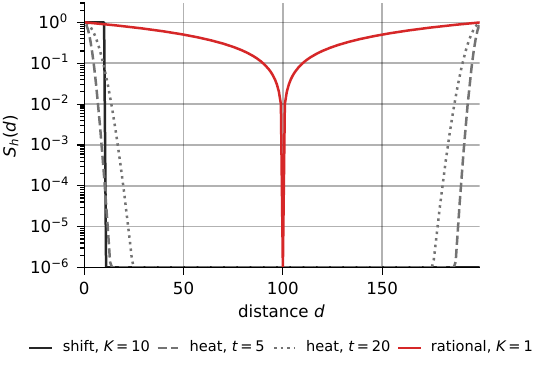}}
    \subfloat[Best filter sensitivity across distances]{%
       \includegraphics[width=0.5\linewidth]{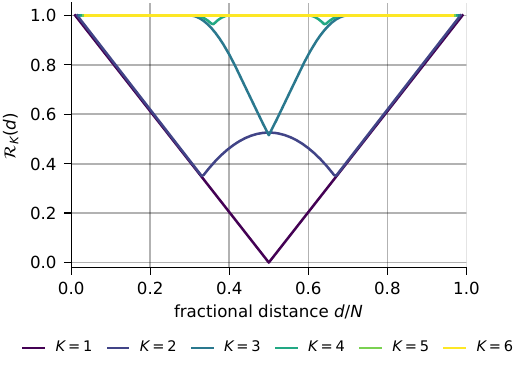}}
      \hfill \\
    \subfloat[$\mathcal{G}(K)$ function]{%
       \includegraphics[width=0.5\linewidth]{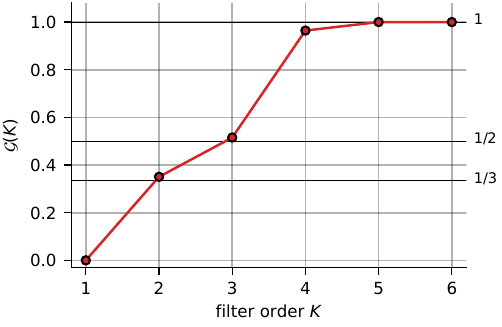}}
  \caption{a) Sensitivity profiles computed for shift, heat ($t=5$), heat ($t=20$), and \filtername ($K=1$) filters. b) Plot of the optimal filter sensitivity across distances. c) Plot of the $\mathcal{G}(K)$ function. Throughout $N$ is set to $200$. }
  \label{app:reach-numerical}
\end{figure}




\vfill
\pagebreak
\section{Accompanying information}
\label{app:accompanying_infos}
\subsection{Spectral conjugate definition}
\begin{definition}[Spectral conjugate]
\label{def:spectral-conjugate-detailed}
Let $\ma A \in \mathbb{C}^{n\times n}$ and consider its Jordan canonical form $\ma A = \ma P(\ma \Lambda + \ma N)\ma P^{-1}$, where $\ma \Lambda$ is diagonal and $\ma N$ is nilpotent, i.e., there exists $k\in\mathbb{N}$ s.t. $\ma N^k=\ma 0$. Then spectral conjugate $\ma A^{\sharp}$ is defined as
\begin{equation}
    \ma A^{\sharp} = \ma P\overline{\ma \Lambda}\ma P^{-1},
    \label{eq:spectral-conjugate-def}
\end{equation}
where $\overline{\ma \Lambda}$ is the complex conjugate of $\ma \Lambda$.
\end{definition}
Interestingly enough, the spectral conjugate $\ma A^{\sharp}$ is different from the Hermitian conjugate $\ma A^{H} = \overline{\ma A^T}$, which is defined as the complex conjugate of the transpose of $\ma A$. The spectral conjugate $\ma A^{\sharp}$ preserves the eigenvectors and generalized eigenvectors of $\ma A$, while only taking the complex conjugate of its eigenvalues. In contrast, the Hermitian conjugate $\ma A^{H}$ does not preserve the eigenvectors of $\ma A$ in general.

\paragraph{Satisfying properties.}
Let $\tau:\mathbb{C} \to \mathbb{C}$ be the map $\tau(z) = \overline{z}$, and functional extension to $\ma A \in \mathbb{C}^{n\times n}$ as $\tau(\ma A) = \ma A^{\sharp}$. Then the spectral conjugate satisfies the following properties:
\begin{itemize}
    \item $\tau(\ma A + \ma B) = \tau(\ma A) + \tau(\ma B)$, \quad if $\ma A$ and $\ma B$ commute,
    \item $\tau(\ma A)B = \ma B\,\tau(\ma A)$,  \quad if $\ma A$ and $\ma B$ commute,
    \item $\tau(\alpha \ma A) = \overline{\alpha}\,\tau(\ma A)$, \quad for any $\alpha \in \mathbb{C}$,
    \item $\tau(\ma P \ma A \ma P^{-1}) = \ma P\,\tau(\ma A)\,\ma P^{-1}$, \quad for any invertible matrix $\ma P$.
\end{itemize}

\paragraph{Computation.}
\label{app:computation-spectral-conjugate}
In practice $\ma A^{\sharp}$ is computed from the complex Schur decomposition $\ma A = \ma Q\ma T\ma Q^{H}$, with $\ma T$ upper triangular. The diagonal is set by conjugating the eigenvalues, $\ma F_{ii} = \overline{\ma T_{ii}}$, and the strictly upper-triangular entries are filled in super-diagonal by super-diagonal using the Parlett recurrence, which enforces the commutation $\ma F\ma T = \ma T\ma F$~\citep{parlett1976recurrence}. For clustered or repeated eigenvalues the scalar recurrence is ill-conditioned and the blocked variant is used instead, with each diagonal block set to $\overline{\mu}\ma I$ for the cluster mean $\mu$ and the off-diagonal blocks recovered from Sylvester equations~\citep{davies2003schur,higham2008functions}. Recomposing, $\ma A^{\sharp} = \ma Q\ma F\ma Q^{H}$, yields the spectral conjugate without forming the Jordan decomposition explicitly.

\subsection{Drazin inverse definition}

\begin{definition}[Drazin inverse]
\label{def:drazin}
Let $\ma A \in \mathbb{C}^{n\times n}$ and let $k = \operatorname{ind}(\ma A)$ denote its index, i.e.\ the smallest non-negative integer such that $\operatorname{rank}(\ma A^{k}) = \operatorname{rank}(\ma A^{k+1})$, equivalently the size of the largest Jordan block associated with the eigenvalue $0$. The Drazin inverse $\ma A^{D}$ is the unique matrix satisfying
\begin{equation}
    \ma A^{D}\ma A\,\ma A^{D} = \ma A^{D},
    \qquad
    \ma A\,\ma A^{D} = \ma A^{D}\ma A,
    \qquad
    \ma A^{k+1}\ma A^{D} = \ma A^{k}.
    \label{eq:drazin-axioms}
\end{equation}
\end{definition}
The commutativity condition $\ma A\ma A^{D} = \ma A^{D}\ma A$ is what distinguishes the Drazin inverse from the Moore--Penrose pseudoinverse $\ma A^{\dagger}$, which satisfies $\ma A\ma A^{\dagger}\ma A = \ma A$ but does not in general commute with $\ma A$.

\paragraph{Spectral characterisation.}
Writing the Jordan decomposition of $\ma A$ with its invertible and nilpotent parts separated,
\begin{equation}
    \ma A = \ma P
      \begin{pmatrix} \ma C & \ma 0 \\ \ma 0 & \ma N \end{pmatrix}
      \ma P^{-1},
    \qquad \ma C \text{ invertible},\quad \ma N^{k} = \ma 0,
    \label{eq:jordan-split}
\end{equation}
the Drazin inverse admits the closed form
\begin{equation}
    \ma A^{D} = \ma P
      \begin{pmatrix} \ma C^{-1} & \ma 0 \\ \ma 0 & \ma 0 \end{pmatrix}
      \ma P^{-1}.
    \label{eq:drazin-spectral}
\end{equation}
Equivalently, $\ma A^{D}$ acts on the spectrum of $\ma A$ through
\begin{equation}
    \lambda \;\longmapsto\;
    \begin{cases}
        \lambda^{-1}, & \lambda \neq 0,\\[2pt]
        0,            & \lambda = 0,
    \end{cases}
    \label{eq:drazin-eigenmap}
\end{equation}
inverting the non-zero part of the spectrum while annihilating the nilpotent part. In particular $\ma A^{D}$ shares the eigenvectors and generalized eigenvectors of $\ma A$, which is precisely why the two operators commute.

\paragraph{Computation.}
In practice $\ma A^{D}$ is obtained from a Schur decomposition $\ma A = \ma Q\ma T\ma Q^{H}$ whose eigenvalues are reordered so that the non-zero ones occupy the leading diagonal block: inverting that block, zeroing the trailing nilpotent block, and recovering the coupling term from a Sylvester equation yields $\ma A^{D}$ without ever forming the Jordan decomposition explicitly~\citep{stanimirovic2015,
campbell2009generalized}. When combined with the computation of the spectral conjugate (Appendix~\ref{def:spectral-conjugate-detailed}), the Schur decomposition is already available and hence only the pointwise scalar inversion and computation of coupling terms is required. Ultimately the computation of $\Lc,\Lu,\Lu(\Lc)^D$ is dominated by the Schur decomposition.

\subsection{Lifting holomorphic function to matrix}
\label{app:holomorphic-to-graphfilters}
The evaluation of a holomorphic function $h$ at a matrix $\ma L$ follows from~\eqref{eq:contour-holomorphic} where $\tfrac{1}{z-\lambda}$ is replaced by the resolvent $(z\ma I - \ma L)^{-1}$~\citep{auscher1997holomorphic}
\begin{equation}
    h(\ma L)
    = \frac{1}{2\pi i}\oint_{\Gamma}
      h(z)\,(z\ma I - \ma L)^{-1}\,dz,
    \label{eq:graph-filter-contour-holomorphic}
\end{equation}
where $\Gamma$ encloses the spectrum $\sigma(\ma L)$ and $(z\ma I - \ma L)$ is invertible for all $z \in \Gamma$.

\subsection{Wirtinger derivative and holomorphicity of polynomial and rational kernels}
\label{app:wirtinger-holomorphicity-poly-ratio}
\paragraph{Wirtinger derivative.} 
Consider a complex number $z=x+jy$ and its conjugate $\bar{z}=x-jy$, the Wirtinger derivatives are defined as follows:
\begin{equation}
    \partial_z = \frac{1}{2}(\partial_x -j \partial_y),\quad \partial_{\bar{z}} = \frac{1}{2}(\partial_x +j \partial_y).
\end{equation}
For notation simplicity, we write $\partial:=\partial_z$ and $\bar{\partial}:=\partial_{\bar{z}}$. In particular the derivatives satisfy $\bar{\partial}z=0$ and $\partial \bar{z}=0$.

\paragraph{Holomorphicity.}
Define a degree $K$ polynomial kernel as:
\begin{equation}
    p_K:\sigma(\ma L)\rightarrow \C,\quad  z \mapsto \sum_{k=0}^K c_kz^k,
\end{equation}
where $c_k\in\C$ are the coefficients. Define a $[K_1,K_2]$ rational kernel as:
\begin{equation}
    r_{K_1,K_2}:\sigma(\ma L)\rightarrow \C,\quad  z \mapsto \frac{p_{K_1}(z)}{p_{K_2}(z)}.
\end{equation}
Since no instance of $\bar{z}$ is in the expressions of $p_K$ and $r_{K_1,K_2}$, the Wirtinger derivatives $\bar{\partial} p_K, \bar{\partial} r_{K_1,K_2}$ are null. Thereby $p_K$ and $r_{K_1,K_2}$ are holomorphic. Consequently, as canonical and Chebyshev polynomials can be expressed in the form $p_K$, they are holomorphic. Similarly, Cayley polynomials \citet{levie2018cayleynets} are special case of the rational kernels, hence they also are holomorphic.

\subsection{Kernels gain for non-diagonalizable operator}
\paragraph{Upper bound on gain of heat kernel}
\label{app:heat-kernel-non-diag}
Consider the heat kernel $h(\lambda)=e^{-t\lambda}$, which is holomorphic since $\bar\partial h=0$, so that $\partial^{k}h=h^{(k)}$ and $h^{(k)}(\lambda)=(-t)^{k}e^{-t\lambda}$. Substituting into \eqref{eq:jordan-filter} for a Jordan block $\ma J_n$ of size $m$ gives
\begin{equation}
    h(\ma J_n)
    = e^{-t\lambda_n}\sum_{k=0}^{m-1}\frac{(-t)^{k}}{k!}\,\ma N_n^{k}
    = e^{-t\lambda_n}
    \begin{pmatrix}
      1 & -t & \tfrac{t^{2}}{2!} & \cdots & \tfrac{(-t)^{m-1}}{(m-1)!}\\
      0 & 1 & -t & \cdots & \tfrac{(-t)^{m-2}}{(m-2)!}\\
      0 & 0 & 1 & \cdots & \tfrac{(-t)^{m-3}}{(m-3)!}\\
      \vdots & \vdots & \vdots & \ddots & \vdots\\
      0 & 0 & 0 & \cdots & 1
    \end{pmatrix}.
\end{equation}
Since $\|\ma N_n^{k}\|_2=1$ (with $\|\cdot\|_2=\sigma_{\text{max}}(\cdot)$ the operator norm) for $k\le m-1$, the gain of the block obeys
\begin{equation}
  \bigl\|h(\ma J_n)\bigr\|_2
  \;\leq\; e^{-t\Re(\lambda_n)}\sum_{k=0}^{m-1}\frac{t^{k}}{k!}
  \;=\; e^{-t\Re(\lambda_n)}\,O\!\left(t^{m-1}\right),
\end{equation}
so defectiveness contributes at most a polynomial factor in $t$, dominated by the exponential whenever $\Re(\lambda_n)>0$. 
\qed

\paragraph{Lower bound on the gain of $e^{\beta \ma L^\uparrow}$.}
\label{app:phase-kernel-non-diag}

In the same way as for the heat kernel, on a Jordan block $\ma J_n$ of size $m$, $\ma L^\uparrow$ acts as
$j\lambda_I[n]\ma I + \tfrac12 \ma N_n$, so $e^{\beta \ma L^\uparrow}\big|_{\ma J_n} = e^{j\beta\lambda_I[n]}\, e^{\beta \ma N_n/2}$,
with inverse $e^{-j\beta\lambda_I[n]}\, e^{-\beta \ma N_n/2}$. Since $\|\ma N_n^k\|_2 = 1$ for $k \le m-1$, its smallest
singular value satisfies
\begin{equation}
    \min_{\|\hat{\vc x}\|_2=1}\big\|e^{\beta \ma L^\uparrow}\big|_{\ma J_n}\hat{\vc x}\big\|_2
    = \big\| e^{-\beta \ma N_n/2} \big\|_2^{-1}
    \;\ge\; \Bigg(\sum_{k=0}^{m-1} \frac{(|\beta|/2)^k}{k!}\Bigg)^{-1}
    \;\ge\; \Big(1 + \tfrac{|\beta|}{2}\Big)^{-(m-1)}.
\end{equation}
The gain is therefore lower bounded by a factor in $\beta$. For frequencies (eigenvalues) in Jordan blocks of size $m=1$, the gain is unitary. For diagonalizable $\ma L$ (i.e. $m = 1$ for all blocks), the gain is unitary for all frequencies. 
\qed

\subsection{Minimal polynomial}
\label{app:minimal_polynomial}
This subsection recalls few propositions and corollaries of \citet{chan2026graph} applicable for diagonalizable $\ma L$ on the minimal polynomial of the dissipative (diffusion) and asymmetric (advection) operators. This determines the maximal layer depth, providing additional guidance for the design of the \model architecture.
\begin{prop} \label{prop:maximal-poly-diff}
    Let $N_C$ be the number of pairs of conjugate eigenvectors, the degree of the minimal polynomial of $\ma L^{\circ}$ is $\text{deg}(m_{\ma L^{\circ}}(\lambda))=N-N_C$.
\end{prop}

\begin{corollary} \label{corr:maximal-poly-diff}
    There exists a set of coefficients $c_k\in\mathbb{C}$ such that the matrix polynomial 
    \begin{equation}
        m_{\ma L^{\circ}}(\ma L^{\circ})=(\ma L^{\circ})^{N-N_C} + \sum_{k=0}^{N-N_C-1} c_k(\ma L^{\circ})^{k}=0.
    \end{equation}
\end{corollary}

\begin{prop} \label{prop:maximal-poly-adv}
    Let $N_Z$ be the number of eigenvalues with zero imaginary part, then it holds that the degree of the minimal polynomial of $\ma L^{\uparrow}$ is $\text{deg}(m_{\ma L^{\uparrow}}(\lambda))=N-N_Z+1$.
\end{prop}

\begin{corollary} \label{corr:maximal-poly-adv}
    There exists a set of coefficients $c_k\in\mathbb{C}$ such that the matrix polynomial 
    \begin{equation}
        m_{\ma L^{\uparrow}}(\ma L^{\uparrow})=(\ma L^{\uparrow})^{N-N_Z+1} + \sum_{k=0}^{N-N_Z} c_k(\ma L^{\uparrow})^{k}=0.
    \end{equation}
\end{corollary}

\subsection{Justification on having both sum and \filtername Graph Filters}

In \citet{chan2026graph}, the sum and \filtername graph kernels are designed to perform different types of filtering. It is important to first note that the image of $h_{\sumf}$ and $h_{\ratf}$ are not mutually inclusive, and span different subspaces. Additionally, authors of \citet{chan2026graph} have shown that $h_{\sumf}$ are better suited for diffusive low-pass filtering, while $h_{\ratf}$ are better suited for advective low-pass filtering and phase manipulation \citet{chan2025hilbert, 11626552}. Finally they can be composed disregarding the order of composition, as the applications of both kernels to $\ma L$ commute.



\subsection{\model Combined architecture: \modelC}
\label{app:TwinSpecGCN-combined}
Following our sum and \filtername graph filters proposed in~\eqref{eq:graph-sum-filter}, one can simply combine both graph filters as they span different subspaces. Denote $\ma X^{(l)}$ the node features of dimension $N \times C$ (nodes times channels) at the $l$-th layer. The $(l+1)$ layer update reads as:
\begin{equation}
\label{eq:TwinSpecGCN-comb}
\begin{split}
    \ma X^{(l+1)} = \psi\Bigl(&\nu\sum_{k=0}^{K_1} T_k(\wLc)\ma X^{(l)}\ma \Theta_k^{(l, \circ)} 
    + (1-\nu)\sum_{k=0}^{K_2} j^k\, T_k(\wLu)\ma X^{(l)}\ma \Theta_k^{(l, \uparrow)} \\
    &+ \rho\sum_{k=0}^{K} j^k\, T_k(\widetilde{\ma Z})\ma X^{(l)}\ma \Theta_k^{(l)} + \ma B^{(l)}\Bigr),
\end{split}
\end{equation}
where $\psi$ is a nonlinear function, $\ma \Theta_{k}^{(l,\circ)}, \ma \Theta_{k}^{(l,\uparrow)}$ and $\ma \Theta_{k}^{(l)}$ are respectively learnable real weight matrices, $\ma B^{(l)}$ is a learnable bias term and learnable parameters $\nu,\rho\in[0,1]$. 

\subsection{Conceptual comparison with advection diffusion graph neural networks}
\label{app:advdiff}
As discussed in Section~\ref{sec:related-works}, for diagonalizable shifts, \eqref{eq:spectral-split} coincides with the graph diffusion and advection operators of \citet{chan2026graph} (separate discussion in following paragraph). Hence, we compare conceptually the graph diffusion and advection operators obtained through our construct for diagonalizable operators with three related works. \citet{berlureau2026advection} reweight the undirected Laplacian by a learned potential field, which remains symmetric and therefore anisotropically diffusive. ADR-GNN \citep{eliasof2024feature} learns the edges of its advection operator, following \citet{chapman2015advection}, whose construction conserves mass rather than emulating transport and is neither skew-symmetric nor read off the input graph. AdvDIFFormer \citep{wu2023supercharging} models diffusion by global attention and advection by message passing along directed edges. In each case the directional term is an architectural addition to a dissipative model and carries no purely imaginary spectrum.

While similar to the work from \citet{chan2026graph}, crucial differences include that our framework applies to non-diagonalizable operators, provide foundations in non-holomorphic calculus, and address problems of oversmoothing and receptive field reach. It is also in practice more stable to use thanks to our contribution in its Chebyshev reparameterization. Finally \citet{chan2026graph} belongs to directed GSP literature and does not tackle topics of neural networks.

\subsection{On polynomial interpolation of the spectral conjugate}
\label{app:finite-polynomial-interpolation}
 
On a finite and fixed graph, any spectral window is exactly realizable by a polynomial
in the shift operator. If $\ma L$ has distinct eigenvalues $\lambda_1,\dots,\lambda_K$ and minimal polynomial of degree $m$, then $h(\ma L)=p_{\ma L}(\ma L)$ for the unique interpolant of degree $\le m-1$ matching the data $\partial^{m}h(\lambda_j)$ \cite[Prop.~5.1, 5.3]{nevanlinna2018non}. Non-holomorphicity of $h$ supposedly could be overcome. We show by explicit computation that the resulting coefficients are nonetheless unusable, whereas for holomorphic they remain stable.
 
In both experiments below we fit a quadratic $p(z)=a_0+a_1z+a_2z^2$ through three eigenvalues by solving the Vandermonde system, and compare two target kernels: the holomorphic $h(z)=z^2+1$ and the conjugation map $\tau(z)=\bar z$, for which $\bar\partial\tau=1\neq0$. We write $z_0=1+i$ and let $\varepsilon>0$.
 
\paragraph{Experiment 1: two close eigenvalues. } Take the eigenvalues $\lambda_1=0$, $\lambda_2=z_0$ and $\lambda_3=z_0+\varepsilon d$ with $|d|=1$, so that $\lambda_3$ lies at distance $\varepsilon$ from $\lambda_2$ in the direction $d$. As $\varepsilon\to0$ the node set converges to $\{0,z_0,z_0\}$ for every choice of $d$.
 
For the holomorphic target kernel the interpolant is exact and independent of both
$\varepsilon$ and $d$,
\begin{equation}
  p(z)=z^2+1,\qquad (a_0,a_1,a_2)=(1,0,1).
\end{equation}
For $\tau$ one obtains $a_0=0$ and
\begin{align}\label{eq:exp1}
  d=1:\quad
  a_1=\frac{i(-\varepsilon-3-i)}{\varepsilon+1+i},\quad
  a_2=\frac{1+i}{\varepsilon+1+i};
  \\
  d=i:\quad
  a_1=\frac{i(-\varepsilon-3+i)}{\varepsilon+1-i},\quad
  a_2=\frac{1+i}{\varepsilon+1-i}.
\end{align}
These coefficients are bounded for small $\varepsilon$, but the denominators differ depending on the direction $d$
\begin{equation}
  \lim_{\varepsilon\to0}(a_1,a_2)=
  \begin{cases}
    (-1-2i,\;1), & d=1,\\
    (\;\;1-2i,\;i), & d=i.
  \end{cases}
\end{equation}
From this example, we see that small numerical inaccuracies from eigensolvers may lead to two completely different filters, caused by the direction of approach for which, smaller distance is ineffective.
 
\paragraph{Experiment 2: three close eigenvalues. }
 Take the eigenvalues $\lambda_1=z_0$, $\lambda_2=z_0+\varepsilon$, $\lambda_3=z_0+i\varepsilon$, so all three nodes lie within $\varepsilon$ of $z_0$. Again the holomorphic target gives $(a_0,a_1,a_2)=(1,0,1)$ for every $\varepsilon$, while for $\tau$
\begin{equation}\label{eq:exp2}
  a_0(\varepsilon)=\frac{2(-1+i)}{\varepsilon},\qquad
  a_1(\varepsilon)=-\frac{i(\varepsilon+4)}{\varepsilon},\qquad
  a_2(\varepsilon)=\frac{1+i}{\varepsilon},
\end{equation}
so that $\max_k|a_k|=(\varepsilon+4)/\varepsilon\sim4\varepsilon^{-1}$.
 
\begin{table}[h]
\centering
\begin{tabular}{lcccc}
\toprule
& $\varepsilon=10^{-1}$ & $10^{-2}$ & $10^{-3}$ & $10^{-4}$\\
\midrule
$\max_k|a_k|$, target $z^2+1$ & $1.0$ & $1.0$ & $1.0$ & $1.0$\\
$\max_k|a_k|$, target $\bar z$ & $4.1\times10^{1}$ & $4.0\times10^{2}$
  & $4.0\times10^{3}$ & $4.0\times10^{4}$\\
\bottomrule
\end{tabular}
\caption{Coefficient magnitude under a cluster of three.}
\end{table}
Beyond the non-uniquely determined interpolation from Experiment.1, the interpolating coefficients may also suffer from arbitrarily large values, making it numerically unstable. While we show ill-defined interpolation and numerical instability for the case of $\tau$ (spectral conjugate), it surmised to be extended to general non-holomorphic kernels.  Either of these examples are realistic especially for non-diagonalizable operators where there necessarily exist an eigenvalue with algebraic multiplicity greater than $1$.

\subsection{Combine with anti-symmetric weight}
\label{app:anti-symmetric weights}

A-DGN~\citep{gravina2023adgn} models each node feature $\vc x_n(t) \in \mathbb R^{C}$ as the state of an ODE whose channel-mixing matrix is constrained to be anti-symmetric (their equation ~4):
\begin{equation}
    \partial_t \vc x_n(t) = \sigma\big((\ma V - \ma V^\top)\,\vc x_n(t) + \Phi(\ma X(t), \mathcal N_n) + \vc b\big),
\end{equation}
where $\ma V \in \mathbb R^{C \times C}$ is learnable and $\Phi$ aggregates over the neighborhood $\mathcal N_n$. Anti-symmetry places the eigenvalues of $\ma V - \ma V^\top$ on the imaginary axis, which is the source of A-DGN's non-dissipative behavior. A forward-Euler discretization with step $\epsilon$ and a small damping $\gamma > 0$ for numerical stability (their equation~5) gives, in matrix form,
\begin{equation}
    \ma X^{(l+1)} = \ma X^{(l)} + \epsilon\, \psi\big(\ma X^{(l)}(\ma V - \ma V^\top - \gamma \ma I) + \Phi(\ma X^{(l)}) + \ma B\big).
\end{equation}

The two constructions act on different axes: A-DGN constrains how channels are mixed, whereas \modelR constrains how information moves across nodes. They therefore combine by using the ratio filter of~\eqref{eq:TwinSpecGCN-ratio} as the aggregation $\Phi$:
\begin{equation}
    \ma X^{(l+1)} = \ma X^{(l)} + \epsilon\, \psi\Big(\ma X^{(l)}(\ma V - \ma V^\top - \gamma \ma I) + \sum_{k=0}^{K} j^k\, T_k(\widetilde{\ma Z})\,\ma X^{(l)}\ma \Theta_k^{(l)} + \ma B^{(l)}\Big),
\end{equation}
where $\ma V$ is the anti-symmetric channel-mixing weight and $\ma \Theta_k^{(l)}$ are the filter weights of \modelR. \modelS and \modelC combine in the same way, by substituting~\eqref{eq:TwinSpecGCN-sum} or~\eqref{eq:TwinSpecGCN-comb} for $\Phi$. Whether the non-dissipativity guarantees of both components hold jointly is not immediate. We leave this analysis, and the empirical evaluation of the combination, to future work.
\vfill
\pagebreak

\section{Experimental details}
\label{app:exp_details}
\subsection{Graph transfer task}
\label{app:graph-transfer-task-details}
\paragraph{General settings.}
The task requires the model to learn a mapping from a source node to a target node on a graph of size $N$. Denote $\vc x[n]\in\mathbb{R}^C$ to be the feature vector at node $n$. Following~\citet{bodnar2021weisfeiler} the node features are set as follows:
\begin{equation}
  \vc x[n] =
  \begin{cases}
    \vc e_k, & \text{if } n\text{ is source node},\\
    \vc 0, & \text{if } n\text{ is target node},\\
    \vc 1 + \vct \epsilon , & \text{otherwise}.
  \end{cases}
\end{equation}
where $\vc e_k$ is set to be a one-hot vector such that $\vct e_k[n] = \delta_{k=n}$ with integer $k\in\{1,\ldots, C\}$ being randomly selected and $\vct \epsilon \sim \mathcal{N}(\vc 0, 0.1 \ma I_N)$. The source node and target node being selected such that the distance between them is $d(n_\text{source}, n_\text{target}) = r$. The ground truth for the target node's feature is set to be the same as the source node's. The model is evaluated through mean squared error (MSE) between the predicted target node feature and the ground truth target node feature. The predicted class is the index with highest value (since all others should be receiving less energy) and it is compared with the ground truth class $k$, leading to an accuracy for each sample. The graphs used for the task are shown in Figure~\ref{app:graph-transfer-task-graphs}. Without loss of generality, in our experiments we select node feature of dimension $C=5$, consider $200$ samples, split into (60\%) training, (20\%) validation and (20\%) testing. The source node is randomly selected, while the target node is selected such that the distance between them is $r=10$ which is the graph diameter.

\begin{figure}[hbt!]
    \centering
       \includegraphics[width=1\linewidth]{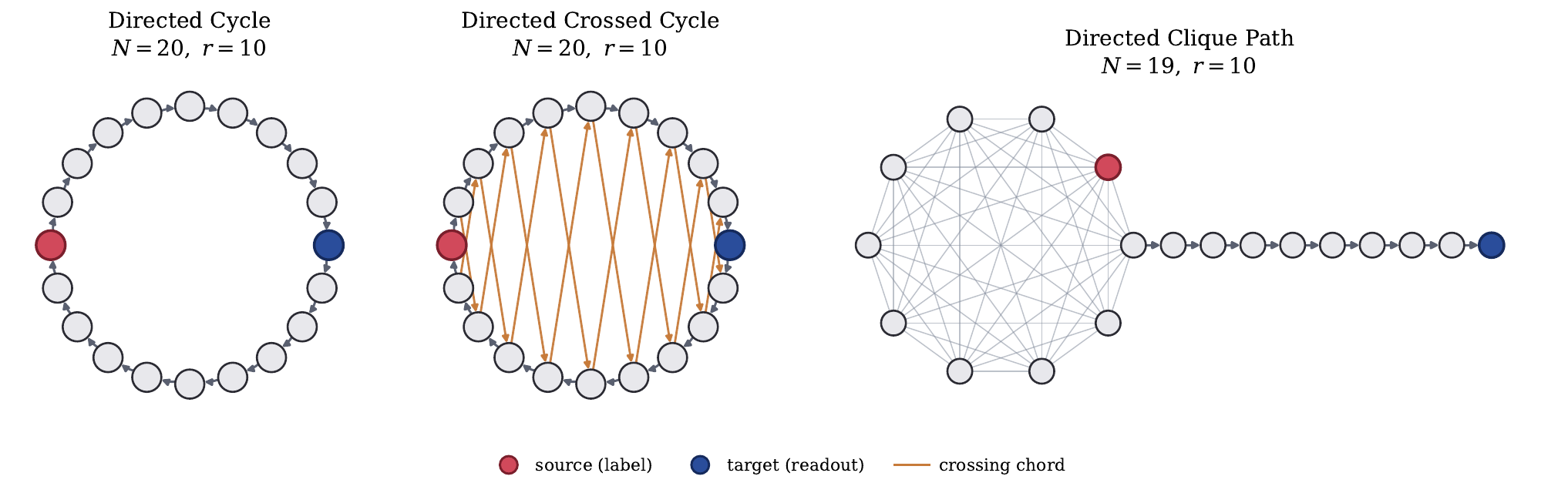}
  \caption{Graphs used for the task (same as in~\citet{gravina2025oversquashing}).}
  \label{app:graph-transfer-task-graphs}
\end{figure}

\paragraph{Details on Section~\ref{sec:receptive-oversmoothing}.}
For depth experiment: source and target are held a fixed distance of $10$ apart, and the sweep varies only network depth. For every topology, model and depth $\ma K\in \{1,\ldots,10\}$ we run an independent hyperparameter search (Appendix~\ref{app:hyperparam-search-graph-transfer}), so no depth is handicapped by settings tuned at another. Each configuration is trained on $200$ sampled graphs split $60\%/20\%/20\%$ into $120$ training, $40$ validation and $40$ test samples under a fixed data seed, with the one-hot label placed on the source node and all other nodes carrying an uninformative constant perturbed by Gaussian noise ($\sigma=0.1$). Training minimizes the mean squared error between the target node's output and the one-hot label, plus a suppression term penalizing the squared output of every non-target node, which prevents the network from spreading the label rather than transporting it. We report accuracy: a prediction counts as correct when its largest coordinate matches the label's. Runs are capped at $300$ epochs and evaluated every $8$ epochs, stopping early after $30$ consecutive evaluations (i.e. $240$ epochs) without an improvement in validation MSE.  Each configuration is scored at its lowest-validation-MSE epoch and the test accuracy is read at that same epoch, so the test split never informs selection. Every reported figure is the mean and standard deviation over $10$ random initializations.

\paragraph{Hyperparameter search.}
\label{app:hyperparam-search-graph-transfer}

The search uses Optuna's Tree-structured Parzen Estimator \citep{optuna_2019}, run independently for each (topology, model, depth) triple at $48$ trials, and selects on validation MSE. Depth is swept exhaustively over $\{1,\ldots,10\}$ rather than sampled, so each model is tuned $480$ times per topology. The searched ranges are: hidden dimension $\in\{4,8,16,32\}$, non-linearity $\in\{\text{ReLU},\text{None}\}$, normalization $\in\{\text{True},\text{False}\}$, bias $\in\{\text{True},\text{False}\}$, and learning rate $\in \{0.001,0.005\}$. Dropout and weight decay are held at $0$ throughout: on a task this small, regularization only slows transport without improving generalization. Three further axes are searched only by the architectures that read them — whether the branch-mixing weights $\alpha$ and $\beta$ are learned or fixed at $0.5$ (employed in Dir-GNN, FaberNet, TwinS-GCN), and MagNet's charge $q\in\{0,0.05,0.10, 0.15, 0.20,0.25\}$. While introduced as the standard combinatorial directed Laplacian, options for operator include $\{\text{left}, \text{combinatorial}, \text{right}, \text{symmetric}\}$ normalized directed Laplacian. Cluster tolerance (Appendix~\ref{app:computation-spectral-conjugate}) was set to $1e^{-9}$, as we expect no distinct eigenvalue clusters. A dash in the table (Table~\ref{tab:hyperparam-search-graph-transfer}) means that the hyperparameter is not relevant to the current architecture. Accuracies are the mean and standard deviation over random initializations. For the same model, when several depths achieve perfect accuracy, the table reports the shallowest, since on a task that many models saturate the meaningful winner is the cheapest one. Corresponding run-times for each of these models are shown in Table~\ref{tab:runtime-graph-transfer-best}. For Dir-GNN, FaberNet, MagNet \model, we consider only the degree of each layer to be $1$ (to fully isolate the effect of depth), meaning that we set $K=1$ for Dir-GNN, FaberNet, MagNet, while for \model, which start at index $K=1$ (programming convention) for the identity matrix, we set $K=2$ to represent degree $1$ polynomial.

\begin{table}[H]
\centering
\caption{Best hyperparameter configuration per topology for the {\bf graph transfer task}, selected over depth and all searched axes by validation MSE. A dash marks an axis the architecture does not take. $\alpha$/$\beta$ report whether the branch-mixing weights are learned; $q$ is MagNet's charge. }
\label{tab:hyperparam-search-graph-transfer}
\resizebox{\columnwidth}{!}{%
\begin{tabular}{clcccccccccc}
\toprule
& \textbf{Model} & \textbf{depth} & \textbf{\#hidden} & \textbf{non-lin.}
  & \textbf{norm.} & \textbf{bias} & \textbf{lr}
  & $\boldsymbol{\alpha}$ & $\boldsymbol{\beta}$ & $\boldsymbol{q}$
  & \textbf{Accuracy}\\
\midrule
\multicolumn{12}{l}{\emph{Directed Cycle}}\\
\midrule
& GCN & $1$ & $8$ & ReLU & No & No & $0.005$ & -- & -- & -- & $0.325 \pm 0.054$\\
& MagNet & $6$ & $4$ & -- & -- & Yes & $0.005$ & -- & -- & $0.05$ & $0.267 \pm 0.031$\\
& FaberNet & $10$ & $32$ & -- & -- & Yes & $0.005$ & Yes & -- & -- & $1.000 \pm 0.000$\\
& Dir-GNN & $10$ & $16$ & None & Yes & No & $0.005$ & Yes & -- & -- & $1.000 \pm 0.000$\\
& ADR-GNN & $2$ & $4$ & -- & -- & -- & $0.005$ & -- & -- & -- & $0.283 \pm 0.024$\\
& AdvDIFFormer & $1$ & $16$ & -- & -- & -- & $0.005$ & -- & -- & -- & $1.000 \pm 0.000$\\
\cmidrule(lr){2-12}
& \modelS & $2$ & $4$ & None & Yes & No & $0.001$ & Yes & -- & -- & $0.258 \pm 0.042$\\
& \modelR & $3$ & $8$ & None & Yes & No & $0.001$ & -- & -- & -- & $1.000 \pm 0.000$\\
& \modelC & $3$ & $16$ & None & Yes & Yes & $0.005$ & No & Yes & -- & $1.000 \pm 0.000$\\
\midrule
\multicolumn{12}{l}{\emph{Directed Crossed Cycle}}\\
\midrule
& GCN & $10$ & $32$ & ReLU & No & Yes & $0.005$ & -- & -- & -- & $0.467 \pm 0.066$\\
& MagNet & $3$ & $8$ & -- & -- & No & $0.001$ & -- & -- & $0.1$ & $0.233 \pm 0.012$\\
& FaberNet & $10$ & $16$ & -- & -- & Yes & $0.005$ & Yes & -- & -- & $1.000 \pm 0.000$\\
& Dir-GNN & $10$ & $32$ & ReLU & No & No & $0.005$ & Yes & -- & -- & $1.000 \pm 0.000$\\
& ADR-GNN & $6$ & $8$ & -- & -- & -- & $0.005$ & -- & -- & -- & $0.300 \pm 0.020$\\
& AdvDIFFormer & $1$ & $32$ & -- & -- & -- & $0.005$ & -- & -- & -- & $1.000 \pm 0.000$\\
\cmidrule(lr){2-12}
& \modelS & $2$ & $32$ & ReLU & Yes & Yes & $0.005$ & Yes & -- & -- & $1.000 \pm 0.000$\\
& \modelR & $2$ & $16$ & None & Yes & No & $0.001$ & -- & -- & -- & $1.000 \pm 0.000$\\
& \modelC & $2$ & $32$ & None & Yes & Yes & $0.001$ & No & No & -- & $1.000 \pm 0.000$\\
\midrule
\multicolumn{12}{l}{\emph{Directed Clique Path}}\\
\midrule
& GCN & $3$ & $16$ & None & Yes & Yes & $0.001$ & -- & -- & -- & $0.275 \pm 0.020$\\
& MagNet & $3$ & $4$ & -- & -- & Yes & $0.005$ & -- & -- & $0.0$ & $0.250 \pm 0.020$\\
& FaberNet & $10$ & $16$ & -- & -- & No & $0.001$ & Yes & -- & -- & $0.258 \pm 0.031$\\
& Dir-GNN & $10$ & $16$ & None & Yes & No & $0.005$ & Yes & -- & -- & $0.767 \pm 0.112$\\
& ADR-GNN & $10$ & $8$ & -- & -- & -- & $0.005$ & -- & -- & -- & $0.275 \pm 0.020$\\
& AdvDIFFormer & $1$ & $32$ & -- & -- & -- & $0.005$ & -- & -- & -- & $1.000 \pm 0.000$\\
\cmidrule(lr){2-12}
& \modelS & $2$ & $32$ & None & Yes & No & $0.005$ & Yes & -- & -- & $1.000 \pm 0.000$\\
& \modelR & $2$ & $16$ & ReLU & Yes & Yes & $0.005$ & -- & -- & -- & $1.000 \pm 0.000$\\
& \modelC & $2$ & $16$ & None & Yes & Yes & $0.001$ & Yes & Yes & -- & $1.000 \pm 0.000$\\
\bottomrule
\end{tabular}%
}
\end{table}

\subsection{Datasets benchmarking}
\paragraph{General settings.}
\label{app:datasets}
The benchmarking datasets used in the experiments are summarized in Table~\ref{tab:benchmarking-datasets-infos}. We evaluate on transductive node classification over directed graphs. Cornell, Texas and Wisconsin are small web-page networks that are strongly heterophilic (edge homophily 0.12, 0.06 and 0.17 respectively), so neighboring nodes usually carry different labels and edge direction is informative. We use the provided ten train/validation/test splits distributed with the datasets. Every configuration is trained on the training nodes for at most 300 epochs, evaluated every 5 epochs, with early stopping after 20 evaluations without validation improvement. Following standard practice, we score each run at the epoch of highest validation accuracy and report the test accuracy at that same epoch, so the test nodes never influence model selection. Each reported figure is the mean and standard deviation of classification accuracy over 30 runs — the ten splits crossed with three random initializations.

\begin{table}[H]
\centering
\caption{Statistics of the node classification benchmarks.}
\label{tab:benchmarking-datasets-infos}
\resizebox{0.6\columnwidth}{!}{
\begin{tabular}{lrrrr}
\toprule
\textbf{Dataset} & \textbf{\#Nodes} & \textbf{\#Edges} & \textbf{\#Feat.} & \textbf{\#Cls.} \\
\midrule
Cornell                & 183      & 295       & 1{,}703 & 5  \\
Texas                  & 183      & 309       & 1{,}703 & 5  \\
Wisconsin              & 251      & 499       & 1{,}703 & 5  \\
Chameleon$^{\dagger}$  & 2{,}277  & 36{,}101  & 2{,}325 & 5  \\
Squirrel$^{\dagger}$   & 5{,}201  & 217{,}073 & 2{,}089 & 5  \\
\bottomrule
\end{tabular}
}
\\[4pt]
{\footnotesize $^{\dagger}$ Original versions. Filtered variants
\citep{platonov2023critical}: Chameleon (890 / 13{,}584),
Squirrel (2{,}223 / 46{,}998).}
\end{table}

Among the results provided in Table~\ref{tab:benchmark-results} here are the ones retrieved from literature:
\begin{itemize}
    \item Results of GCN, GAT on Wisconsin, Cornell, Texas are taken from~\citet{eliasof2024feature} and on Chameleon-F, Squirrel-F, comes from~\citet{platonov2023critical}.
    \item Results of MLP on Chameleon-F, Squirrel-F, comes from~\citet{platonov2023critical} (ResNet).
\end{itemize}
All remaining entries are computed from re-implementation following original repository. Disclaimer: the results reproduced for MLP, MagNet, ADR-GNN on Wisconsin, Cornell, Texas may differ from accuracies found in their respective original paper~\citet{zhang2021magnet},~\citet{eliasof2024feature}. We have noticed in their open repository that the employed Wisconsin, Cornell, Texas datasets were from the pre-2022 geom-gcn~\citep{pei2020geomgcn} release while in our experiments we employ the latest release~\citep{geomgcn_repo}, hence we reproduce the results on these datasets.

\paragraph{Hyperparameter search.} 
\label{app:hyperparam-search-benchmarking}
Selected hyperparameters for the benchmarking datasets are reported in Table~\ref{tab:hyperparam-search-benchmarking}. The search uses Optuna's TPE sampler (multivariate), and configurations are selected on mean validation accuracy. The shared search space is: depth $\in\{2,3,4,5\}$, hidden dimension $\in\{32,64,128,256\}$, dropout $\in\{0,0.2,0.4,0.5,0.6\}$,  normalization and bias $\in\{\text{True},\text{False}\}$, learning rate $\in\{0.001,0.005,0.01,0.05\}$, weight decay $\in\{0,5.10^{-4},5.10^{-3}\}$, input feature scaling $\in\{\text{raw},\text{row},\text{std}\}$, jumping knowledge $\in\{\text{None},\text{max},\text{cat}\}$ and, for spectral models, order $K\in\{1,2,3,4\}$. The non-linearity is fixed to ReLU. Each model also searches its own parameters: MagNet's charge $q\in\{0,0.05,\dots,0.25\}$; the Dir-GNN and FaberNet mixing weight $\alpha\in\{0,0.5,1\}$ and whether it is learned;  FaberNet's zero-order term; the ADR-GNN step size $\in\{0.02,0.05,0.1,0.2\}$; and the AdvDIFFormer heads  $\in\{1,2,4\}$ and order $K_{\mathrm{ord}}\in\{1,2,3\}$. For \model, the search covers the branch-mixing weights $\alpha,\beta$, the Laplacian normalization $\in\{\text{left},\text{combinatorial},\text{right},\text{symmetric}\}$, and the eigenvalue-clustering tolerance $\in\{\text{default},10^{-6},10^{-3}\}$. Operators are rescaled by their maximum absolute entry. Next we discuss few differences in protocol: on WebKB, \model uses $72$ trials, against $96$ for MLP, MagNet, Dir-GNN and FaberNet and $72$ for ADR-GNN and AdvDIFFormer to find hyperparameters. On Chameleon-F and Squirrel-F, \model uses $32$ trials, against $248$ and $48$ respectively. \model thus has the smallest budget on every dataset. Its validation score, however, averages over all $10$ splits, whereas the baselines average over $3$ ($5$ for the attention-based baselines on Chameleon-F and Squirrel-F). Finally, for Texas and Squirrel-F the \model search was widened to depth $1$, $K\in\{5,\dots,8\}$, weight decay $10^{-2}$ and patience $50$, which accounts for the settings in Table~\ref{tab:hyperparam-search-benchmarking} that are not within the shared search space.

\begin{table}[H]
\centering
\caption{Hyperparameter search results for {\bf benchmarking datasets}. A dash marks an axis the architecture does not take. $\alpha$/$\beta$ report whether the branch-mixing weights are learned; $q$ is MagNet's charge. Cluster tol are $1e^{-9}$ unless specified.}
\label{tab:hyperparam-search-benchmarking}
\resizebox{\columnwidth}{!}{%
\begin{tabular}{clcccccccccccclc}
\toprule
& \textbf{Model} & \textbf{depth} & \textbf{\#hidden} & \textbf{dropout} & \textbf{norm.} & \textbf{bias} & \textbf{lr} & \textbf{w.d.} & \textbf{feat.} & \textbf{JK} & $\boldsymbol{K}$ & \textbf{learn} $\boldsymbol{\alpha}$ & \textbf{learn} $\boldsymbol{\beta}$ & \textbf{model-specific} & \textbf{Accuracy}\\
\midrule
\multicolumn{16}{l}{\emph{Chameleon-F}}\\
\midrule
& MagNet & $3$ & $128$ & $0.2$ & -- & Yes & $0.005$ & $0.005$ & raw & -- & $1$ & -- & -- & $q{=}0$ & $0.424 \pm 0.035$\\
& Dir-GNN & $5$ & $128$ & $0.5$ & No & Yes & $0.005$ & $0.0005$ & row & cat & -- & No & -- & $\alpha{=}0.5$ & $0.441 \pm 0.029$\\
& FaberNet & $4$ & $128$ & $0.6$ & -- & Yes & $0.01$ & $0.0005$ & row & max & $1$ & No & -- & $\alpha{=}0.5$, zero-ord.: No & $0.430 \pm 0.033$\\
& AdvDIFFormer & $4$ & $128$ & $0.4$ & -- & -- & $0.05$ & $0$ & std & -- & -- & -- & -- & heads $2$, $K_{\mathrm{ord}}{=}1$ & $0.415 \pm 0.036$\\
& ADR-GNN & $4$ & $64$ & $0.6$ & -- & -- & $0.01$ & $0.005$ & std & -- & -- & -- & -- & step $0.05$ & $0.461 \pm 0.041$\\
\cmidrule(lr){2-16}
& \modelS & $2$ & $256$ & $0.6$ & Yes & No & $0.001$ & $0$ & std & cat & $2$ & Yes & -- & right & $0.423 \pm 0.028$\\
& \modelR & $2$ & $256$ & $0.5$ & No & Yes & $0.001$ & $0.005$ & std & cat & $4$ & -- & -- & none, tol $10^{-6}$ & $0.442 \pm 0.032$\\
& \modelC & $2$ & $128$ & $0.4$ & Yes & No & $0.01$ & $0$ & std & none & $4$ & No & Yes & none, tol $10^{-3}$ & $0.416 \pm 0.036$\\
\midrule
\multicolumn{16}{l}{\emph{Squirrel-F}}\\
\midrule
& MagNet & $5$ & $128$ & $0.2$ & -- & Yes & $0.05$ & $0$ & raw & -- & $2$ & -- & -- & $q{=}0.05$ & $0.442 \pm 0.014$\\
& Dir-GNN & $4$ & $64$ & $0.5$ & No & Yes & $0.005$ & $0.0005$ & row & cat & -- & No & -- & $\alpha{=}0.5$ & $0.443 \pm 0.022$\\
& FaberNet & $3$ & $64$ & $0.6$ & -- & Yes & $0.001$ & $0.005$ & raw & None & $1$ & No & -- & $\alpha{=}0.5$, zero-ord.: No & $0.451 \pm 0.021$\\
& AdvDIFFormer & $2$ & $64$ & $0.6$ & -- & -- & $0.005$ & $0$ & std & -- & -- & -- & -- & heads $1$, $K_{\mathrm{ord}}{=}1$ & $0.414 \pm 0.016$\\
& ADR-GNN & $5$ & $64$ & $0.4$ & -- & -- & $0.05$ & $0.005$ & std & -- & -- & -- & -- & step $0.2$ & $0.441 \pm 0.020$\\
\cmidrule(lr){2-16}
& \modelS & $3$ & $64$ & $0.6$ & Yes & No & $0.001$ & $0.0005$ & std & cat & $4$ & Yes & -- & left, tol $10^{-3}$ & $0.396 \pm 0.015$\\
& \modelR & $2$ & $32$ & $0.6$ & Yes & No & $0.005$ & $0.005$ & std & None & $2$ & -- & -- & left, tol $10^{-3}$ & $0.402 \pm 0.018$\\
& \modelC & $2$ & $128$ & $0.6$ & Yes & Yes & $0.005$ & $0$ & std & none & $5$ & Yes & Yes & left & $0.401 \pm 0.016$\\
\midrule
\multicolumn{16}{l}{\emph{Texas}}\\
\midrule
& MLP & $4$ & $32$ & $0.4$ & -- & Yes & $0.01$ & $0.005$ & raw & -- & -- & -- & -- & -- & $0.822 \pm 0.043$\\
& MagNet & $2$ & $256$ & $0.2$ & -- & Yes & $0.01$ & $0.005$ & raw & -- & $1$ & -- & -- & $q{=}0.1$ & $0.853 \pm 0.046$\\
& Dir-GNN & $2$ & $64$ & $0.6$ & Yes & No & $0.05$ & $0.005$ & raw & cat & -- & Yes & -- & $\alpha{=}0$ & $0.705 \pm 0.065$\\
& FaberNet & $2$ & $64$ & $0.5$ & -- & Yes & $0.005$ & $0.005$ & raw & cat & $2$ & No & -- & $\alpha{=}1$, zero-ord.: Yes & $0.832 \pm 0.053$\\
& AdvDIFFormer & $3$ & $32$ & $0.2$ & -- & -- & $0.05$ & $0.0005$ & raw & -- & -- & -- & -- & heads $4$, $K_{\mathrm{ord}}{=}1$ & $0.818 \pm 0.064$\\
& ADR-GNN & $2$ & $64$ & $0.2$ & -- & -- & $0.05$ & $0.005$ & raw & -- & -- & -- & -- & step $0.02$ & $0.822 \pm 0.052$\\
\cmidrule(lr){2-16}
& \modelS & $2$ & $64$ & $0.5$ & Yes & Yes & $0.01$ & $0.0005$ & raw & max & $3$ & No & -- & left & $0.846 \pm 0.034$\\
& \modelR & $1$ & $128$ & $0.5$ & Yes & No & $0.01$ & $0.0005$ & row & cat & $8$ & -- & -- & comb. & $0.841 \pm 0.039$\\
& \modelC & $2$ & $256$ & $0.2$ & Yes & Yes & $0.01$ & $0.01$ & raw & cat & $8$ & Yes & No & left, tol $10^{-3}$ & $0.854 \pm 0.047$\\
\midrule
\multicolumn{16}{l}{\emph{Wisconsin}}\\
\midrule
& MLP & $4$ & $64$ & $0.4$ & -- & Yes & $0.01$ & $0.005$ & raw & -- & -- & -- & -- & -- & $0.852 \pm 0.050$\\
& MagNet & $2$ & $32$ & $0.4$ & -- & Yes & $0.01$ & $0.005$ & raw & -- & $1$ & -- & -- & $q{=}0.15$ & $0.834 \pm 0.040$\\
& Dir-GNN & $5$ & $64$ & $0.4$ & Yes & No & $0.05$ & $0.0005$ & raw & max & -- & No & -- & $\alpha{=}0.5$ & $0.674 \pm 0.060$\\
& FaberNet & $2$ & $256$ & $0.5$ & -- & Yes & $0.001$ & $0.005$ & raw & cat & $2$ & No & -- & $\alpha{=}1$, zero-ord.: Yes & $0.795 \pm 0.052$\\
& AdvDIFFormer & $2$ & $256$ & $0.4$ & -- & -- & $0.01$ & $0.0005$ & raw & -- & -- & -- & -- & heads $1$, $K_{\mathrm{ord}}{=}2$ & $0.765 \pm 0.048$\\
& ADR-GNN & $3$ & $64$ & $0.2$ & -- & -- & $0.05$ & $0.005$ & raw & -- & -- & -- & -- & step $0.02$ & $0.845 \pm 0.045$\\
\cmidrule(lr){2-16}
& \modelS & $2$ & $64$ & $0.2$ & Yes & No & $0.005$ & $0.0005$ & row & cat & $2$ & Yes & -- & right & $0.871 \pm 0.041$\\
& \modelR & $2$ & $256$ & $0.2$ & Yes & Yes & $0.005$ & $0.005$ & raw & cat & $3$ & -- & -- & right & $0.874 \pm 0.040$\\
& \modelC & $2$ & $64$ & $0$ & Yes & No & $0.01$ & $0.005$ & raw & max & $2$ & Yes & No & right & $0.868 \pm 0.038$\\
\midrule
\multicolumn{16}{l}{\emph{Cornell}}\\
\midrule
& MLP & $3$ & $64$ & $0.2$ & -- & Yes & $0.01$ & $0.005$ & raw & -- & -- & -- & -- & -- & $0.759 \pm 0.039$\\
& MagNet & $2$ & $32$ & $0.2$ & -- & Yes & $0.01$ & $0.005$ & raw & -- & $1$ & -- & -- & $q{=}0.25$ & $0.733 \pm 0.042$\\
& Dir-GNN & $2$ & $128$ & $0.2$ & Yes & No & $0.05$ & $0.005$ & std & None & -- & No & -- & $\alpha{=}0.5$ & $0.598 \pm 0.069$\\
& FaberNet & $3$ & $32$ & $0.2$ & -- & No & $0.01$ & $0.005$ & raw & cat & $4$ & No & -- & $\alpha{=}0$, zero-ord.: Yes & $0.726 \pm 0.047$\\
& AdvDIFFormer & $2$ & $128$ & $0.4$ & -- & -- & $0.01$ & $0.0005$ & raw & -- & -- & -- & -- & heads $2$, $K_{\mathrm{ord}}{=}1$ & $0.722 \pm 0.055$\\
& ADR-GNN & $2$ & $128$ & $0$ & -- & -- & $0.005$ & $0.005$ & raw & -- & -- & -- & -- & step $0.02$ & $0.714 \pm 0.044$\\
\cmidrule(lr){2-16}
& \modelS & $2$ & $64$ & $0.5$ & Yes & Yes & $0.01$ & $0.0005$ & raw & max & $3$ & No & -- & sym. & $0.751 \pm 0.032$\\
& \modelR & $2$ & $64$ & $0.5$ & Yes & Yes & $0.01$ & $0.0005$ & raw & max & $3$ & -- & -- & right & $0.768 \pm 0.030$\\
& \modelC & $2$ & $64$ & $0.5$ & Yes & Yes & $0.01$ & $0.0005$ & raw & max & $3$ & No & No & right & $0.757 \pm 0.040$\\
\bottomrule
\end{tabular}%
}
\end{table}

\vfill
\pagebreak

\section{Models Runtime}

\paragraph{Graph transfer task.}
We report the runtime of the models used in the graph transfer task. The runtime are provided for all different graphs. The runtime are measured on Apple M1 Max. Table~\ref{tab:runtime-graph-transfer-best} shows the minimal needed runtime/parameters of each model to solve the graph transfer task. Table~\ref{tab:runtime-graph-transfer-comparable} shows for comparable configuration, the models runtime/parameters.

\vfill

\begin{table}[H]
\centering
\caption{Runtime and parameter count per topology for the \textbf{graph transfer
task} as specified in Appendix~\ref{app:graph-transfer-task-details}, for the \textbf{best configuration} selected in Table~\ref{tab:hyperparam-search-graph-transfer}. Train time is per-epoch, averaged over $100$ epochs (train repeats is $5$) and inference time (infer repeats is $50$) is per forward pass over the full graph ($40$ graphs/batch). Build time is provided specifically to \model models which require a one-time precomputing, per graph, of the operators. Accuracy is reported for reference. }
\label{tab:runtime-graph-transfer-best}
\resizebox{\columnwidth}{!}{%
\begin{tabular}{clccccc}
\toprule
& \textbf{Model} & \textbf{\#Params} & \textbf{Train (ms/epoch)} & \textbf{Inference (ms)} & \textbf{Build (ms)} &\textbf{Accuracy}\\
\midrule
\multicolumn{7}{l}{\emph{Directed Cycle}}\\
\midrule
& GCN & $25$ & $0.67 \pm 0.34$ & $0.09 \pm 0.00$ & -- & $0.325 \pm 0.054$\\
& MagNet & $269$ & $9.12 \pm 0.28$ & $1.80 \pm 0.01$ & -- & $0.267 \pm 0.031$\\
& FaberNet & $39{,}110$ & $26.94 \pm 0.78$ & $4.44 \pm 0.00$ & -- & $1.000 \pm 0.000$\\
& Dir-GNN & $4{,}417$ & $10.58 \pm 0.06$ & $1.96 \pm 0.00$ & -- & $1.000 \pm 0.000$\\
& ADR-GNN & $345$ & $11.67 \pm 0.52$ & $1.61 \pm 0.01$ & -- & $0.283 \pm 0.024$\\
& AdvDIFFormer & $1{,}317$ & $3.66 \pm 0.23$ & $0.48 \pm 0.14$ & -- & $1.000 \pm 0.000$\\
\cmidrule(lr){2-7}
& \modelS & $161$ & $2.88 \pm 0.43$ & $0.43 \pm 0.04$ & $0.0$ & $0.258 \pm 0.042$\\
& \modelR & $288$ & $2.60 \pm 0.14$ & $0.34 \pm 0.00$ & $0.0$ & $1.000 \pm 0.000$\\
& \modelC & $2{,}719$ & $6.40 \pm 0.71$ & $1.17 \pm 0.07$ & $0.0$ & $1.000 \pm 0.000$\\
\midrule
\multicolumn{7}{l}{\emph{Directed Crossed Cycle}}\\
\midrule
& GCN & $8{,}805$ & $14.49 \pm 0.54$ & $2.89 \pm 0.16$ & -- & $0.467 \pm 0.066$\\
& MagNet & $421$ & $4.92 \pm 0.09$ & $0.98 \pm 0.00$ & -- & $0.233 \pm 0.012$\\
& FaberNet & $10{,}342$ & $23.55 \pm 0.17$ & $4.39 \pm 0.14$ & -- & $1.000 \pm 0.000$\\
& Dir-GNN & $17{,}025$ & $10.47 \pm 0.26$ & $1.80 \pm 0.14$ & -- & $1.000 \pm 0.000$\\
& ADR-GNN & $3{,}213$ & $53.38 \pm 0.69$ & $9.21 \pm 0.20$ & -- & $0.300 \pm 0.020$\\
& AdvDIFFormer & $4{,}677$ & $4.45 \pm 0.08$ & $0.63 \pm 0.00$ & -- & $1.000 \pm 0.000$\\
\cmidrule(lr){2-7}
& \modelS & $1{,}429$ & $2.98 \pm 0.01$ & $0.52 \pm 0.06$ & $0.5$ & $1.000 \pm 0.000$\\
& \modelR & $320$ & $1.73 \pm 0.04$ & $0.23 \pm 0.00$ & $0.9$ & $1.000 \pm 0.000$\\
& \modelC & $2{,}142$ & $3.88 \pm 0.05$ & $0.79 \pm 0.06$ & $0.9$ & $1.000 \pm 0.000$\\
\midrule
\multicolumn{7}{l}{\emph{Directed Clique Path}}\\
\midrule
& GCN & $453$ & $3.79 \pm 0.09$ & $0.69 \pm 0.00$ & -- & $0.275 \pm 0.020$\\
& MagNet & $161$ & $4.84 \pm 0.18$ & $0.96 \pm 0.00$ & -- & $0.250 \pm 0.020$\\
& FaberNet & $9{,}702$ & $21.28 \pm 0.15$ & $4.01 \pm 0.00$ & -- & $0.258 \pm 0.031$\\
& Dir-GNN & $4{,}417$ & $10.55 \pm 0.08$ & $1.95 \pm 0.00$ & -- & $0.767 \pm 0.112$\\
& ADR-GNN & $5{,}293$ & $113.09 \pm 0.73$ & $21.52 \pm 0.85$ & -- & $0.275 \pm 0.020$\\
& AdvDIFFormer & $4{,}677$ & $4.92 \pm 0.11$ & $0.65 \pm 0.01$ & -- & $1.000 \pm 0.000$\\
\cmidrule(lr){2-7}
& \modelS & $1{,}281$ & $2.81 \pm 0.04$ & $0.52 \pm 0.06$ & $0.2$ & $1.000 \pm 0.000$\\
& \modelR & $362$ & $1.71 \pm 0.03$ & $0.27 \pm 0.03$ & $0.4$ & $1.000 \pm 0.000$\\
& \modelC & $1{,}088$ & $4.30 \pm 0.04$ & $0.78 \pm 0.06$ & $0.4$ & $1.000 \pm 0.000$\\
\bottomrule
\end{tabular}%
}
\end{table}

\begin{table}[H]
\centering
\caption{Runtime and parameter count per topology for the \textbf{graph transfer
task} as specified in Appendix~\ref{app:graph-transfer-task-details}. The count is under an \textbf{equivalent configuration} (hidden=32, depth=4, polynomial degree (K)=1,  dropout=0, bias=True, non-linearity=ReLU). Remaining information is as in Table~\ref{tab:runtime-graph-transfer-best}.}
\label{tab:runtime-graph-transfer-comparable}
\resizebox{0.66\columnwidth}{!}{
\begin{tabular}{clccc}
\toprule
& \textbf{Model} & \textbf{\#Params} & \textbf{Train (ms/epoch)} & \textbf{Inference (ms)}\\
\midrule
\multicolumn{5}{l}{\emph{Directed Cycle}}\\
\midrule
& GCN & $2{,}469$ & $8.93 \pm 1.80$ & $1.02 \pm 0.24$\\
& MagNet & $6{,}917$ & $17.34 \pm 1.30$ & $3.01 \pm 0.68$\\
& FaberNet & $27{,}205$ & $38.95 \pm 0.34$ & $6.77 \pm 0.60$\\
& Dir-GNN & $4{,}938$ & $9.33 \pm 0.21$ & $1.40 \pm 0.12$\\
& ADR-GNN & $30{,}181$ & $134.74 \pm 3.63$ & $22.85 \pm 4.63$\\
& AdvDIFFormer & $17{,}445$ & $34.32 \pm 0.39$ & $5.04 \pm 1.11$\\
\cmidrule(lr){2-5}
& \modelS & $9{,}876$ & $11.96 \pm 0.25$ & $2.05 \pm 0.72$\\
& \modelR & $4{,}938$ & $10.03 \pm 4.00$ & $1.61 \pm 0.97$\\
& \modelC & $14{,}814$ & $22.52 \pm 2.21$ & $3.22 \pm 1.56$\\
\midrule
\multicolumn{5}{l}{\emph{Directed Crossed Cycle}}\\
\midrule
& GCN & $2{,}469$ & $12.64 \pm 0.43$ & $2.20 \pm 0.89$\\
& MagNet & $6{,}917$ & $19.03 \pm 1.85$ & $4.64 \pm 1.57$\\
& FaberNet & $27{,}205$ & $49.51 \pm 3.24$ & $7.91 \pm 1.93$\\
& Dir-GNN & $4{,}938$ & $10.00 \pm 0.33$ & $1.79 \pm 0.81$\\
& ADR-GNN & $30{,}181$ & $193.44 \pm 16.22$ & $35.31 \pm 7.38$\\
& AdvDIFFormer & $17{,}445$ & $41.97 \pm 2.84$ & $5.76 \pm 2.45$\\
\cmidrule(lr){2-5}
& \modelS & $9{,}876$ & $13.75 \pm 0.83$ & $2.21 \pm 0.37$\\
& \modelR & $4{,}938$ & $8.47 \pm 0.47$ & $1.11 \pm 0.34$\\
& \modelC & $14{,}814$ & $20.97 \pm 1.74$ & $3.28 \pm 0.91$\\
\midrule
\multicolumn{5}{l}{\emph{Directed Clique Path}}\\
\midrule
& GCN & $2{,}469$ & $17.46 \pm 0.99$ & $2.81 \pm 1.18$\\
& MagNet & $6{,}917$ & $20.83 \pm 3.74$ & $4.27 \pm 1.24$\\
& FaberNet & $27{,}205$ & $47.63 \pm 6.61$ & $7.50 \pm 1.42$\\
& Dir-GNN & $4{,}938$ & $9.92 \pm 0.80$ & $1.56 \pm 0.41$\\
& ADR-GNN & $30{,}181$ & $303.24 \pm 15.23$ & $53.27 \pm 23.50$\\
& AdvDIFFormer & $17{,}445$ & $33.62 \pm 1.16$ & $4.62 \pm 1.47$\\
\cmidrule(lr){2-5}
& \modelS & $9{,}876$ & $14.47 \pm 0.36$ & $2.58 \pm 0.95$\\
& \modelR & $4{,}938$ & $9.20 \pm 0.43$ & $1.71 \pm 0.34$\\
& \modelC & $14{,}814$ & $18.87 \pm 0.24$ & $3.49 \pm 0.94$\\
\bottomrule
\end{tabular}%
}
\end{table}

\paragraph{Benchmarking datasets.}
We report the runtime of the models used in the benchmarking of datasets. The runtime are measured on Apple M1 Max. Table~\ref{tab:runtime-benchmark-best} shows the runtime/parameters used by each model to solve the benchmarking task. Table~\ref{tab:runtime-benchmark-comparable} shows for comparable configuration, the models runtime/parameters.

\begin{table}[H]
\centering
\caption{Runtime and parameter count, illustrated with \textbf{Wisconsin} dataset \textbf{benchmarking of datasets}, for the \textbf{best configuration} selected in Table~\ref{tab:hyperparam-search-benchmarking}. Train time is per-epoch, averaged over $100$ epochs (train repeats is $5$) and inference time (infer repeats is $50$) is per forward pass over the full graph. Accuracy is reported to distinguish relevance across models. }
\label{tab:runtime-benchmark-best}
\resizebox{0.7\columnwidth}{!}{%
\begin{tabular}{clcccc}
\toprule
& \textbf{Model} & \textbf{\#Params} & \textbf{Train (ms/epoch)} & \textbf{Inference (ms)} &\textbf{Accuracy}\\
\midrule
& MLP      & $117{,}701$   & $1.44 \pm 0.26$  & $0.19 \pm 0.05$  & $0.852 \pm 0.050$\\
& MagNet   & $111{,}429$   & $6.78 \pm 0.25$  & $3.76 \pm 0.74$  & $0.834 \pm 0.040$\\
& Dir-GNN  & $251{,}072$   & $4.15 \pm 0.12$  & $1.49 \pm 0.20$  & $0.674 \pm 0.060$\\
& FaberNet & $6{,}029{,}317$ & $55.83 \pm 2.13$ & $23.56 \pm 6.79$ & $0.795 \pm 0.052$\\
& AdvDIFFormer & $1{,}095{,}941$ & $13.45 \pm 0.50$ & $2.67 \pm 0.36$ & $0.765 \pm 0.048$\\
& ADR-GNN      & $197{,}125$     & $18.80 \pm 0.58$ & $8.50 \pm 0.29$ & $0.845 \pm 0.045$\\
\cmidrule(lr){2-6}
& \modelS & $452{,}993$ & $4.19 \pm 0.81$ & $1.17 \pm 0.37$ & $0.871 \pm 0.041$\\
& \modelR & $1{,}508{,}613$ & $16.27 \pm 0.16$ & $8.10 \pm 1.72$ & $0.874 \pm 0.040$\\
& \modelC & $678{,}849$ & $5.23 \pm 0.11$ & $1.80 \pm 0.36$ & $0.868 \pm 0.038$\\
\bottomrule
\end{tabular}%
}
\end{table}

\begin{table}[H]
\centering
\caption{Build time (ms) per \model model, per graph. }
\label{tab:build-time-benchmark}
\resizebox{0.8\columnwidth}{!}{%
\begin{tabular}{clccccc}
\toprule
& \textbf{Model} & \textbf{Cornell} & \textbf{Texas} & \textbf{Wisconsin} & \textbf{Chameleon-F} &\textbf{Squirrel-F}\\
\midrule
& \modelS & $20$ & $10$ & $50.0$ & $1,230$ & $22,540$\\
& \modelR & $50$ & $40$ & $90.0$ & $3,810$ & $58,540$\\
& \modelC & $50$ & $40$ & $90.0$ & $3,810$ & $58,540$\\
\bottomrule
\end{tabular}%
}
\end{table}

\begin{table}[H]
\centering
\caption{Runtime and parameter count, illustrated with \textbf{Wisconsin} dataset \textbf{benchmarking of datasets}, for \textbf{equivalent configuration} (hidden=32, depth=4, dropout=0, K=2, bias=True). Train time is per-epoch, averaged over $50$ epochs (train repeats is $3$) and inference time (infer repeats is $50$) is per forward pass over the full graph. }
\label{tab:runtime-benchmark-comparable}
\resizebox{0.66\columnwidth}{!}{%
\begin{tabular}{clccc}
\toprule
& \textbf{Model} & \textbf{\#Params} & \textbf{Train (ms/epoch)} & \textbf{Inference (ms)}\\
\midrule
& MLP & $56{,}805$ & $1.27 \pm 0.44$ & $0.17 \pm 0.04$\\
& Dir-GNN & $113{,}610$ & $1.72 \pm 0.09$ & $0.39 \pm 0.00$\\
& MagNet & $173{,}157$ & $10.67 \pm 0.28$ & $5.85 \pm 0.30$\\
& FaberNet & $692{,}677$ & $14.61 \pm 0.27$ & $6.29 \pm 0.12$\\
& ADR-GNN & $84{,}517$ & $15.03 \pm 0.14$ & $6.03 \pm 0.67$\\
& AdvDIFFormer & $79{,}973$ & $6.93 \pm 0.36$ & $1.54 \pm 0.00$\\
\cmidrule(lr){2-5}
& \modelS & $227{,}220$ & $2.97 \pm 0.07$ & $0.68 \pm 0.00$\\
& \modelR & $113{,}610$ & $1.58 \pm 0.08$ & $0.41 \pm 0.03$\\
& \modelC & $340{,}830$ & $4.16 \pm 0.02$ & $1.07 \pm 0.03$\\
\bottomrule
\end{tabular}%
}
\end{table}

\vfill
\pagebreak

\section{Additional experiment for receptive field reach}
\label{app:reach-generalization}



\subsection{Graph signal transfer on cycle graph}
\label{app:graph-signal-cycle}
\paragraph{Settings.} To gain further insight into the results of Section~\ref{sec:receptive-oversmoothing}, we conduct a long-range dependency experiment on the directed cycle graph $\mathcal{C}_N$ with $N=200$ nodes. The goal of this task is to find the filter coefficients to transfer a localized signal from a source node to a target node at various distances. In particular, we study the target reconstruction error of the different filters (\eqref{eq:graph-sum-filter}) underlying \model, as their behavior determines the receptive field of the model.

We start with a localized signal defined as a gaussian centered at an arbitrarily chosen source node, and we select a target node such that the number of hops (shortest path distance) between them is $r$. The starting signal and the ground truth target signal are thus respectively defined as
\begin{equation}
    \vc x_{\text{source}}[n] = e^{-\frac{(n - n_{\text{source}})^2}{2\sigma^2}},
    \qquad
    \vc x_{\text{target}}[n] = e^{-\frac{(n - n_{\text{target}})^2}{2\sigma^2}},
    \qquad \text{with } d(n_{\text{source}}, n_{\text{target}}) = r,
\end{equation}
where $\sigma$ is the bump's width, and $d$ shortest path distance from $n_\text{source}$ to $n_\text{target}$.

First, we generate different pairs of source and target nodes, spanning all distances up to the graph radius $r=100$ (half the graph), and construct signals with $\sigma=2$. Then, for each fixed order $k \in \{10,20,50\}$, we estimate the coefficients of the sum and ratio filters $h_{\sumf}$ and $h_{\ratf}$ introduced in Section \ref{sec:non-holomorphicity} and the polynomial filter $h$ of the same order. 
The coefficients are optimized for the normalized MSE (NMSE) between the filtered source signal and the ground truth target signal, where 
\begin{equation}
    \text{NMSE}(\vc x_{\text{target}}, \hat{\vc x}_{\text{source},k})= \frac{\bigl\| \vc x_{\text{target}}- \hat{\vc x}_{\text{source},k}\bigr\|_2^2}{\bigl\| \vc x_{\text{target}}\bigr\|_2^2} 
    \label{eq:mse-reconstruct-filter}
\end{equation}
where $ \hat{\vc x}_{\text{source},k} = \hat{h}_{i,k}(\ma L)\,{\vc x}_{\text{source}}$, and $\hat{h}_{i,k}$ is the estimated $k$-th order filter of class $h_i \in \{h, h_\sumf,h_\ratf\}$. The results are shown in Figure~\ref{app:long-range-dep-mse-all-orders}.

\begin{figure}[H]
    \centering
       \includegraphics[width=.95\linewidth]{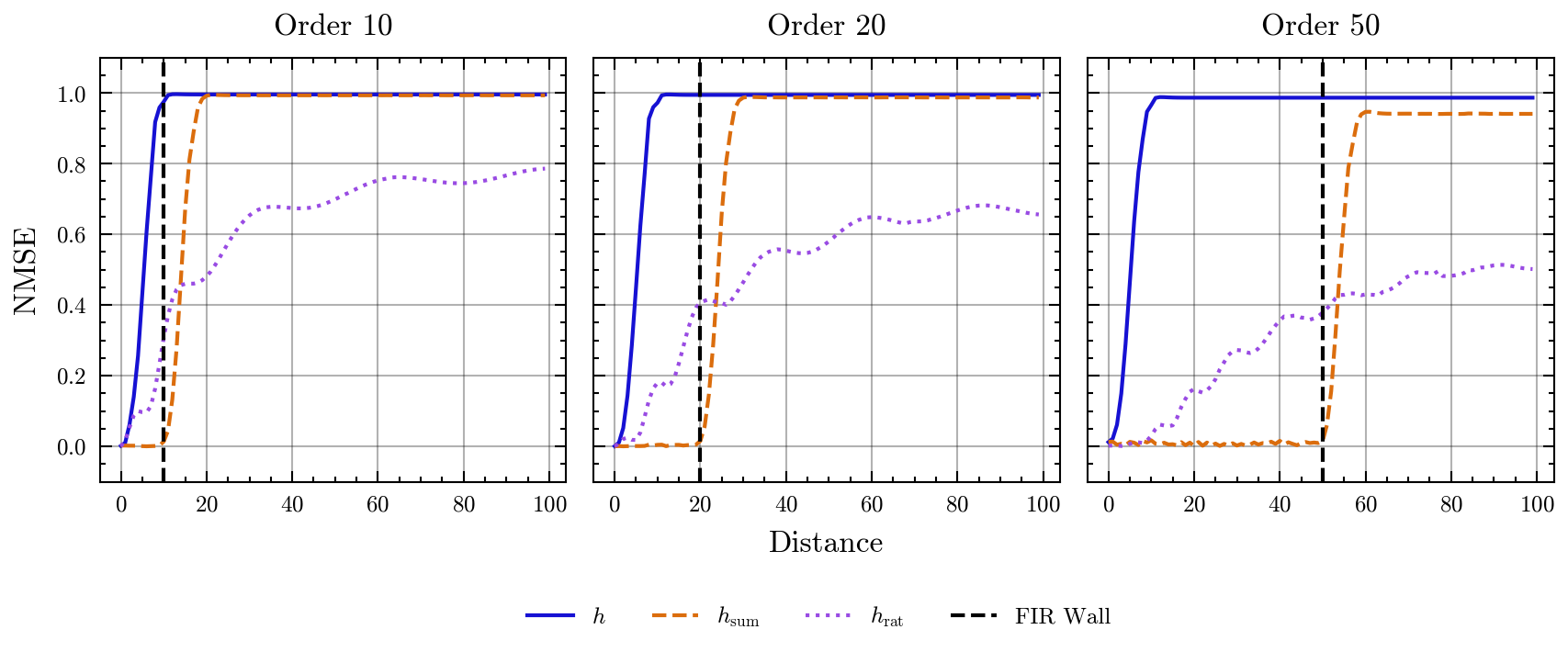}
  \caption{Long-range dependency experiment on cycle graph with increasing orders.}
  \label{app:long-range-dep-mse-all-orders}
\end{figure}


\paragraph{Results.}
We clearly observe the effect of the FIR wall in Figure~\ref{app:long-range-dep-mse-all-orders}, where the kernel $h$ generally struggles to reach $0$ NMSE (perfect transfer) even when its order (degree) is higher than the distance between the source and the target, which in the Figure is shown as percentage of the graph radius). The kernel $h_\sumf$ also performs marginally better than $h$, consistent with the capacity of leveraging $\Lu$ stated in Section~\ref{subsubsec:oversmoothing}, preventing oversmoothing. Hence in contrast to $h$, $h_\sumf$ keeps the initial signal localized, reducing NMSE. Finally, $h_\ratf$ overcomes this wall as even when its order is largely lower than the distance, the signal transfer still happen ($\text{NMSE}< 1$). It is however performing slightly worse than polynomial kernels when its order is larger than the distance, as expected from Theorem~\ref{thm:reach-gap}. 

\subsection{Graph signal transfer on temperature graph}
\paragraph{Settings.} Similarly to the experiment in Section ~\ref{app:graph-signal-cycle}, we conduct a long-range dependency experiment on a temperature graph with $N=212$ nodes, built as described in Figure~\ref{app:temp-graph} \citep{chan2026graph}. This time around, we set the source node randomly on the east side of the graph (source regions as the wind goes from east to west) and a target node randomly on the west side. The starting signal and the ground truth target signal are thus respectively defined as
\begin{equation}
\vc x_{\text{source}}[n] = e^{-\frac{(\vc c[n] - \vc c[n_\text{source}])^2}{2\sigma^2}}, \quad 
\vc x_{\text{target}}[n] = e^{-\frac{(\vc c[n] - \vc c[n_\text{target}])^2}{2\sigma^2}}, \qquad \text{with } d(n_{\text{source}}, n_{\text{target}}) = r,
\end{equation}
with $\vc c[n]$ being the coordinates of node $n$, $\sigma$ the bump's width, and $d$ the shortest path distance from $n_\text{source}$ to $n_\text{target}$. We generate $10$ different pairs of source and target nodes, leading to $10$ pairs of signals with $\sigma=1$. An example of source-target pair is shown in Figure~\ref{app:temperature_task}. The coefficients are optimized for the NMSE between the filtered source signal and the ground truth target signal (\eqref{eq:mse-reconstruct-filter}). The results are shown in Figure~\ref{app:illustration_task_result} and~\ref{app:long-range-dep-mse-all-orders-temperature-graph}. We observe in particular that $h_{\ratf}$ breaks the FIR wall and is able to better reach the target node despite a similar order than other kernels. With increasing order, the performance of $h_{\ratf}$ improves, but is surpassed by $h_{\sumf}$. 

\begin{figure}[H]
    \centering
    \captionsetup[subfigure]{justification=centering}
  \subfloat[Measured wind velocity field\label{app:velocity-field-temp}]{%
       \includegraphics[width=0.4\linewidth]{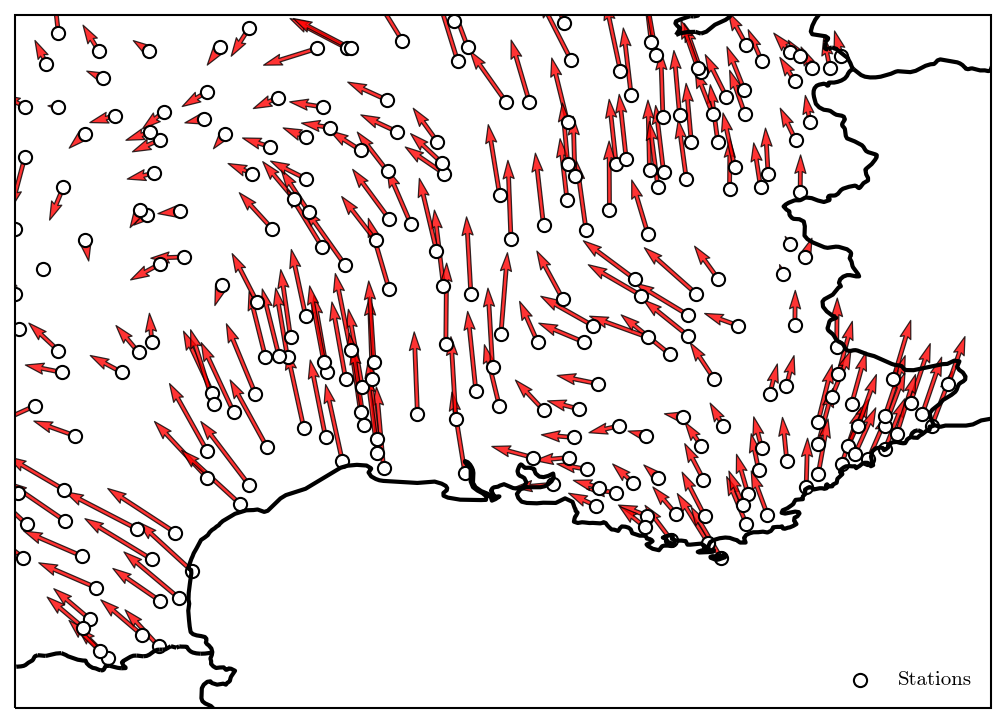}}
    \subfloat[Derived temperature graph \label{app:temp-graph}]{%
       \includegraphics[width=0.4\linewidth]{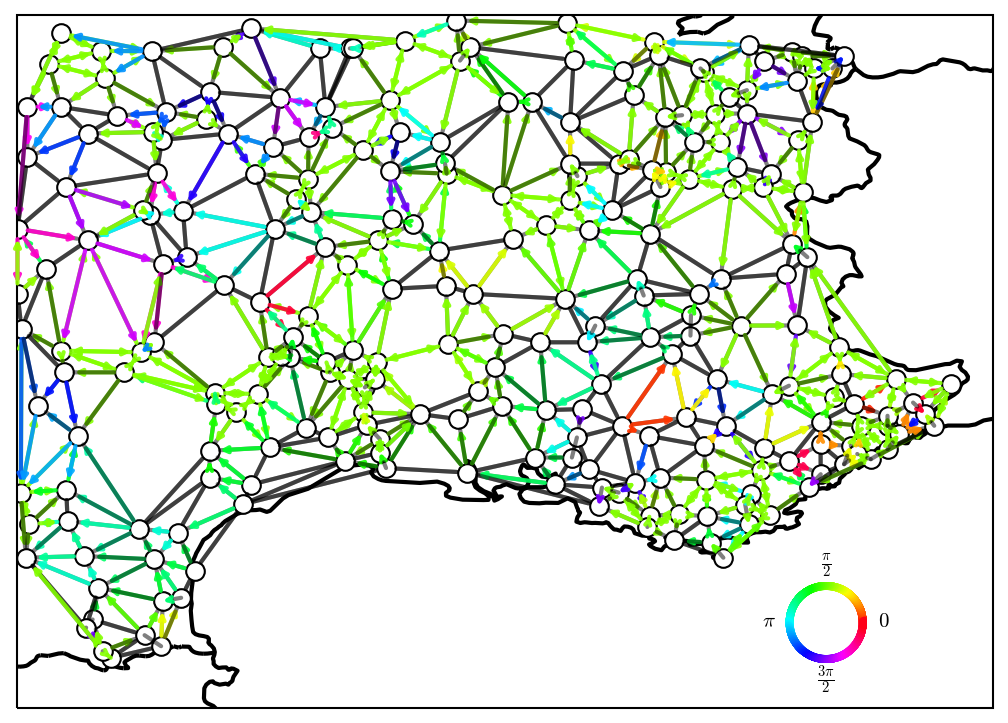}}
       \\
    \subfloat[Signal centered on source and target nodes used for the temperature displacement task \label{app:temperature_task}]{%
       \includegraphics[width=0.8\linewidth]{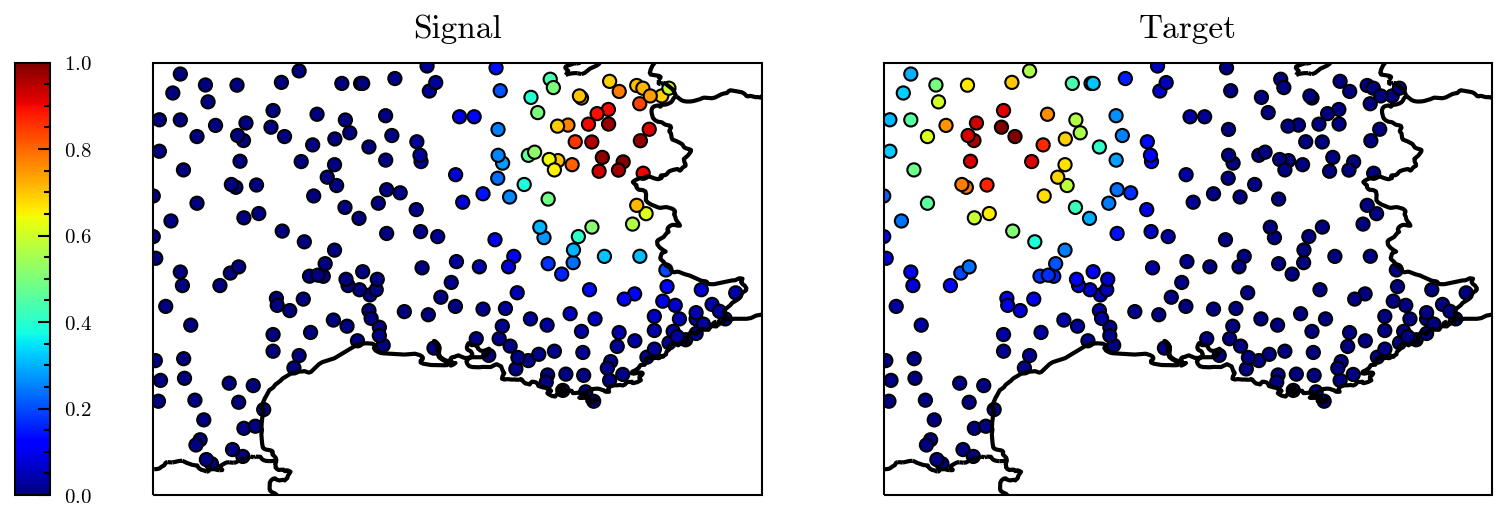}}
  \caption{a) Measured wind velocity field. b) Derived temperature graph. c) Signal centered on source and target nodes used for the temperature displacement task.}
\end{figure}

\begin{figure}[H]
    \centering
    \captionsetup[subfigure]{justification=centering}
    \subfloat[Estimated target signal (filter order: $K=10$) \label{app:illustration_task_result}]{%
       \includegraphics[width=1\linewidth]{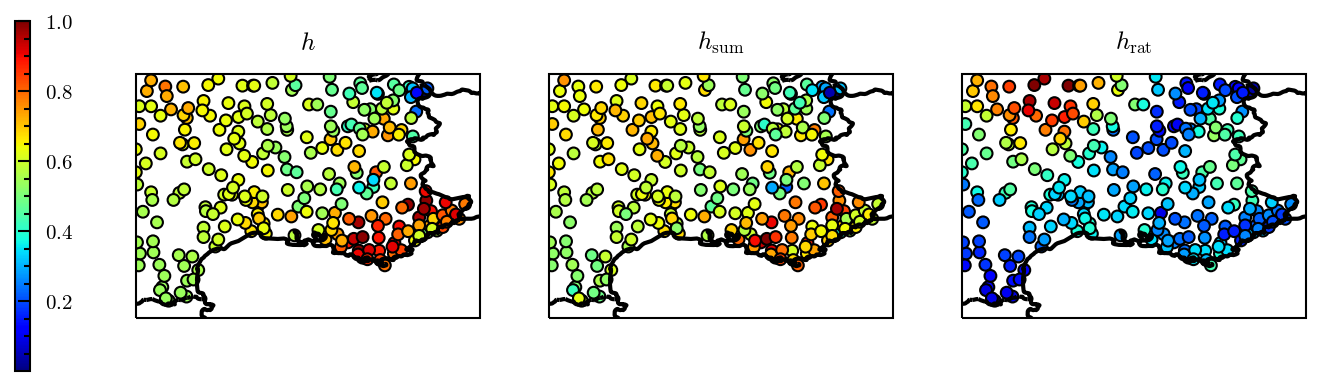}}
       \\
    \subfloat[NMSE for target reconstruction as a function of filter order \label{app:long-range-dep-mse-all-orders-temperature-graph}]{%
       \includegraphics[width=0.75\linewidth]{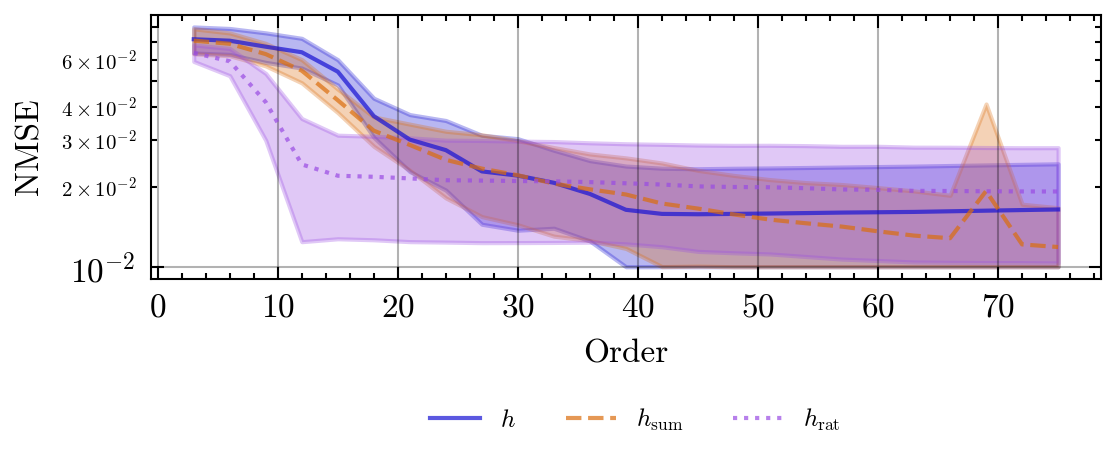}}
  \caption{a) Estimated target signal using kernels $h,h_\sumf,h_\ratf$ with $K=10$. b) Mean NMSE across the 10 (source,target) pairs for estimated target signal as a function of filter order.}
\end{figure}

\paragraph{Results.} 
In this experiment we show similar takeaway as in Appendix~\ref{app:graph-signal-cycle} but instead for a real graph, that is a temperature graph with direction given by the wind. We see once more a better NMSE performance for $h_\ratf$ on low filter order regime w.r.t. $h$ and $h_\sumf$. On the other hand, $h_\sumf$ is able to achieve an overall better NMSE than $h$.

\end{document}

%% file: math_commands.tex
\usepackage{amsmath,amsfonts,bm}

\def\eqref#1{equation~\ref{#1}}

\def\1{\bm{1}}

\DeclareMathAlphabet{\mathsfit}{\encodingdefault}{\sfdefault}{m}{sl}
\SetMathAlphabet{\mathsfit}{bold}{\encodingdefault}{\sfdefault}{bx}{n}

\newcommand{\R}{\mathbb{R}}



%% file: macros.tex
\newcommand{\vc}[1]{{\mathbf #1}}
\newcommand{\vct}[1]{\pmb{#1}}
\newcommand{\ma}[1]{{\mathbf #1}}

\newcommand{\uparrowcirc}{
\tikz[baseline=0ex]{
\draw[->] (0,-0.08) -- (0,0.2);
\draw (0,0.05) circle (0.06);
}}

\newcommand{\C}{\mathbb{C}}

\newcommand{\Lc}{{\mathbf L}^{\circ}}
\newcommand{\Lu}{{\mathbf L}^{\uparrow}}
\providecommand{\Lrat}{\ma{L}^{\uparrowcirc}}
\newcommand{\wLc}{\widetilde{\mathbf L}^{\circ}}
\newcommand{\wLu}{\widetilde{\mathbf L}^{\uparrow}}
\newcommand{\lR}{\lambda_{\mathrm{R}}}
\newcommand{\lI}{\lambda_{\mathrm{I}}}

\newcommand{\sumf}{\text{sum}}
\newcommand{\ratf}{\text{rat}}

\newcommand{\model}{\textsc{TwinS-GCN}\xspace}
\newcommand{\modellayer}{\textsc{TwinS}\xspace}
\newcommand{\modelS}{\modellayer (S)\xspace}   
\newcommand{\modelR}{\modellayer (R)\xspace}   
\newcommand{\modelC}{\modellayer (C)\xspace}   

\newcommand{\filtername}{ratio\xspace}   

\providecommand{\R}{\mathbb{R}}

\newtheorem{theorem}{Theorem}
\newtheorem{lemma}[theorem]{Lemma}

\theoremstyle{definition}
\newtheorem{definition}{Definition}[section]

\newtheorem{prop}{Proposition}
\theoremstyle{definition}

\newtheorem{corollary}{Corollary}